\documentclass[11pt]{article}
\pdfoutput=1 
\usepackage{arxiv}
\usepackage{url}            % simple URL typesetting
\usepackage{booktabs}       % professional-quality tables
\usepackage[american]{babel}
\usepackage{comment}

\usepackage{amsmath}%
\usepackage{amsthm}%
\usepackage{amsfonts}%
\usepackage{amssymb}%
\usepackage{graphicx}%
\usepackage[title]{appendix}%
\usepackage{xcolor}%
\usepackage{textcomp}%
\usepackage{manyfoot}%
\usepackage{booktabs}%
\usepackage{enumerate}%
\usepackage{mathrsfs}%
\usepackage{braket}%
\usepackage{algorithmicx}%
\usepackage{algpseudocode}%
\usepackage{listings}%
\usepackage{caption}
\usepackage{subcaption}

\usepackage{hyperref}
\newtheorem{theorem}{Theorem}

\newtheorem{definition}[theorem]{Definition}

\newtheorem{lemma}[theorem]{Lemma}

\newtheorem{proposition}[theorem]{Proposition}
\newtheorem{remark}[theorem]{Remark}

\newcommand{\mbf}[1]{\mathbf{#1}}
\newcommand{\sbf}[1]{\boldsymbol{#1}}
\renewcommand{\sc}[1]{{#1}}
\usepackage{makecell}

\usepackage{placeins}
\usepackage{capt-of}

\DeclareMathOperator*{\argmax}{arg\,max}

\newcommand{\vasp}{{\textsc ProjSe}} % in case we still want to easily change the name of the algorithm

\begin{document}

\title{Scalable variable selection for two-view learning tasks with
projection operators 
% \\ {\textit arXiv version 0.24}
}

\author{
  Sandor Szedmak \\ 
  Department of Computer Science\\
  Aalto University\\
  Espoo, Finland \\
 \texttt{sandor.szedmak@aalto.fi} \\
 \And
  Riikka Huusari \\ 
  Department of Computer Science\\
  Aalto University\\
  Espoo, Finland \\
 \texttt{riikka.huusari@aalto.fi} \\
 \And
  Tat Hong Duong Le \\ 
  Department of Computer Science\\
  Aalto University\\
  Espoo, Finland \\
 \texttt{duong.h.le@aalto.fi} \\
 \And
  Juho Rousu \\ 
  Department of Computer Science\\
  Aalto University\\
  Espoo, Finland \\
 \texttt{juho.rousu@aalto.fi} \\
}

% \author*[1]{\fnm{Sandor} \sur{Szedmak} \email{sandor.szedmak.aalto.fi}} 
% \author[1]{\fnm{Riikka} \sur{Huusari} \email{riikka.huusari@aalto.fi}}
% \author[1]{\fnm{Tat Hong} \sur{Duong Le} \email{duong.h.le@aalto.fi}}
% \author[1]{\fnm{Juho} \sur{Rousu} \email{juho.rousu@aalto.fi}}

% \affil*[1]{\orgdiv{Department of Computer Science}, \orgname{Aalto
% University}, \orgaddress{P.O. Box 11000 (Otakaari 1B)}, \city{Espoo}, \postcode{FI-00076}, \country{Finland}  
% }

\date{}
\maketitle

\begin{abstract}

In this paper we propose a novel variable selection method for two-view settings, or for vector-valued supervised learning problems. 
Our framework is able to handle extremely large scale selection tasks, where number of data samples could be even millions. 
In a nutshell, our method performs variable selection by iteratively selecting variables that are highly correlated with the output variables, but which are not correlated with the previously chosen variables.   
To measure the correlation, our method uses the concept of projection operators and their algebra.  
With the projection operators the relationship, correlation, between sets of input and output variables can also be expressed by kernel functions, thus nonlinear correlation models can be exploited as well.
%Notably, our proposed variable selection by projection (\vasp) is invariant on the output variable representation whenever they span the same subspace. 
We experimentally validate our approach, showing on both synthetic and real data its scalability and the relevance of the selected features. 
     
\end{abstract}

%\blfootnote{The authors wish to acknowledge the financial support by Academy of Finland through the grants 334790 (MAGITICS), 339421 (MASF) and 345802 (AIB), as well as the Global Programme by Finnish Ministry of Education and Culture}
% in here: https://www.springer.com/journal/10994/submission-guidelines
% it is said that "Acknowledgments of people, grants, funds, etc. should be placed in a separate section on the title page. The names of funding organizations should be written in full."
% I don't find any mention in the template about acknowledgements, so I have no idea how they are supposed to be formatted!

\section{Introduction}

% intro should start with vector-valued outputs and two-view learning methods (vv, structured pred, cca..) then problem of interpretability in this setting and comparative scarcity of feature selection methods

% Due to ever increasing amounts of data acquisition, it is becoming increasingly rare to encounter simple scalar-valued supervised learning tasks for machine learning applications. 

% Machine learning settings going beyond the simple scalar-valued prediction tasks are ...
% Many applications for machine learning arising from real-world data result in challenging settings. 
Vector-valued, or more generally structured output learning tasks arising from various domains have attracted much research attention in recent years \cite{micchelli2005learning,deshwal2019learning,brogat2022vector}. %  how many/what exactly here? don't want to add too many to keep to the page limit
For both supervised but also unsupervised learning approaches, multi-view data has been of interest \cite{hotelling1936relations,xu2013survey,quang2013unifying}. 
% paper by minh ha quang is also vector-valued in addition to multi-view
% 
Despite many successful approaches for various multi-view and vector-valued learning settings, including interpretability to these models has received less attention.  
While there are various feature selection and dimensionality reduction methods either for scalar-valued learning tasks, or unsupervised methods for data represented in a single view~\cite{zebari2020comprehensive, Li_2017, Anette2018ABO,
RePEc:eee:csdana:v:143:y:2020:i:c:s016794731930194x}, there is scarcity of methods suitable for when data is represented in two views, or arises from a vector-valued learning task. 
From the point of view  of interpretability, especially feature selection methods are advantageous over dimensionality reduction since the relevant features are directly obtained as a result and not given only in (linear) combinations.

% Feature or variable selection is a large set of both supervised and unsupervised methods to reduce data dimensionality by subsampling the available variables~\cite{zebari2020comprehensive}.
% Including interpretability to vector-valued, or two-view learning problems is much of an open research question. 

% Here the relevant works from literature should be written in, with the caveats why they are not perfect for the task at hand.
% Essentially I would bring here the methods from related works

Recently, some feature selection methods have been proposed for structured output learning tasks.
\cite{brouard2022feature} proposed kernel-based non-linear feature selection model relying on sparsity regularization. %Notably, for the supervised setting this method considers the output data via kernelized representation, and is thus applicable to learning tasks with general outputs. \textbf{This is in stark contrast to our kernelized extension, where we consider kernelized representation of the output variables.}
Another supervised feature selection approach based on kernel methods is introduced in~\cite{10.5555/2503308.2343691}, this one relying
instead on forward- and backward-selection ideology. % Again, a kernel is defined on the space of output data samples. 
In addition,~\cite{https://doi.org/10.48550/arxiv.2110.05852} discusses
feature selection in conjunction with kernel-based models, obtaining
sparsity implicitly via loss function without explicit
regularization term. An alternative, spline based, approach to the non-linear feature
selection is proposed by \cite{JMLR:v18:17-178}.  
These methods, relying on the kernel evaluations between data samples for both inputs and outputs, tend not to scale very well to large sample sizes.

In this paper, we introduce a novel variable selection approach for  vector-valued, or two-view learning tasks, including CCA. 
Our method is based on efficient iterative computation of projections of input variables to the vector space intersection between the  space spanned by the output variables and the one of the previously selected input variables. In this space, the input variables are then selected by a correlation-based criterion.
% 
% 
%
% I think it should be made much clearer than our use of kernels is very different from the of the methods mentioned here: (1) Our definitions actually amount to a co-variance matrix between the kernel-induced features of the output i.e. if kernel is linear the entries of the kernel matrix denote the uncentered covariances. (2) our kernel is on the outputs unlike these works
Going one step further, we also exploit a kernel-based representation of the variables, allowing us to capture complex, nonlinear relationships. Here, we consider the kernelised representation of the variables instead of data samples -- in essence, we model the co-variance on the features in the Hilbert space induced by the kernel. Notably, both input and output features are captured with the same kernelisation. 
This is in stark contrast to other proposed kernel-based feature selection approaches in literature, where separate kernels are used for data samples in input and output spaces~\cite{brouard2022feature, 10.5555/2503308.2343691, https://doi.org/10.48550/arxiv.2110.05852}.
We can more readily draw comparisons to canonical correlation analysis (CCA) and its kernelized version, where the correlations are computed between
two sets of variables instead 
of pairs of individual ones \cite{DeBie2005}.
% \textit{[what will be demonstrated about CCA in the end? ]}

For many approaches, scalability in feature selection can be a challenge for when the data dimensionality is extremely large.  
Some supervised linear feature selection models adapted to this setting are proposed in~\cite{10.5555/1577069.1755853,aghazadeh2018mission, 10.1371/journal.pcbi.1010180}. We note, that all these methods are for the supervised setting, but with scalar-valued output variables. While scalability w.r.t the feature dimensionality is often considered due to motivations arising from fat data, the scalability to large sample sizes is less focused on.
Traditionally, kernelized algorithms, while powerful, are very poorly scalable due to the dependence to the kernel matrix, especially if its inverse is required. 
Contrary to the usual, by leveraging the recursive formulation of our algorithm and a trick with singular value decomposition on the variable representation, our approach is extremely scalable to large sample sizes - which we also demonstrate experimentally in Section~\ref{sec:experiments}.

% \textit{Due to the properties of projection operators, our approach promotes invariance in the selection procedure: the selected explanatory variables (or input features) depend only on the subspace spanned by the response variables (output features), and are independent on any transformation on the response variables that would span the same subspace. These transformations can be singular or even nonlinear, as long as they are automorphisms of the output space. 
% This also makes our method highly persistent on certain types of noise.  [illustrated? or a small example?]
% %
% %Very often in literature the representation of the response variable is not considered but is used as is, without discussion on how the chosen representation or data transformations might affect the results. 
% %In comparison to this, our approach considers a subspace spanned by the outputs, with the projection operator. Consequently, it is invariant w.r.t. transformations preserving the subspace spanned by the outputs.  
% }  
% \textbf{invariance discussion has to be figured out}

To summarize, our main contributions in this paper are as follows:
\begin{itemize}
\item we propose 
projective selection (\vasp) algorithm, a novel approach for variable selection for vector-valued or two-view learning problems that is based on projection operators. In \vasp{} the result of the feature selection only depends on the subspace spanned by the outputs, not on the specific values (invariance). 
\item our proposed iterative method offers high scalability even for the kernelised formulation capturing non-linearities in the data, due to a trick with singular value decomposition applied to the feature representation.
\item we experimentally validate the proposed approach, showing both relevance of the selected features and the efficiency of the algorithm. 
\end{itemize}

\begin{table}[b]
\caption{Some of the frequently used notation in this paper.}\label{tb:notation}
\centering
\begin{tabular}{@{}l@{\quad}p{4.6in}@{}}
$\mathcal{H}$  & is a Hilbert space - unless otherwise noted, it has
                 finite dimension $d$, in which case $\mathcal{H} = \mathbb{R}^{d}$. \\
% $\braket{.,.}$ & is the inner product in $\mathcal{H}$  \\
% $||.||$ & is the inner product induced norm in $\mathcal{H}$. \\
$\oplus$ & denotes the direct sum of subspaces \\  
$\mbf{I}$  & is the identity operator acting on $\mathcal{H}$. \\
$\mathcal{L}$  & is an arbitrary subspace of $\mathcal{H}$. \\
$\mathcal{L}_{\mbf{X}}$  & is a subspace of $\mathcal{H}$, spanned by
                          the columns of matrix $\mbf{X}$. \\
$\mathcal{L}^{\perp}$  & is a subspace of $\mathcal{H}$, the
                         orthogonal complement of $\mathcal{L}$. \\ 
$\mbf{P}_{\mathcal{L}}$   & is an orthogonal projection operator into subspace
                  $\mathcal{L}$, $\mbf{P}_{\mathcal{L}}: \mathcal{H} \rightarrow
                            \mathcal{L}$ \\
$\mathcal{L}_{\mbf{P}}$ & is the subspace corresponding to the
                          projection operator $\mbf{P}$. \\  
$\mbf{P}_{\mathcal{L}^{\perp}}$   & is an orthogonal projection operator into
                          the orthogonal complement of subspace
                          $\mathcal{L}$. \\
$\mbf{P}_{\mbf{X}}$  & is an orthogonal projection operator into the
                       subspace of $\mathcal{H}$, spanned by 
                          the columns of matrix $\mbf{X}$. It is the
                       same as $\mbf{P}_{\mathcal{L}_{X}}$.\\
$\mbf{P}_{\mbf{X}^{\perp}}$  & is an orthogonal projection operator into the
                       subspace of $\mathcal{H}$ orthogonal to the
                               subspace spanned by 
                          the columns of matrix $\mbf{X}$. It is the
                       same as $\mbf{P}_{\mathcal{L}_{X}^{\perp}}$.\\
$\mbf{A}^{+}$ & denotes the Moore-Penrose inverse of matrix
                $\mbf{A}$. \\
                  %\cite{RPenrose1955}. \\ % <- probably unncecessary
                  %citation 
$[n]$   & is a short hand notation for the set $\{1,\dots,n\}.$ \\ 
$\mbf{A}\circ \mbf{B}$ & denotes pointwise(Hadamard) product of 
                         matrices $\mbf{A}$ and $\mbf{B}$. \\
$\mbf{A}^{\circ n }$ & is the pointwise power of matrix $\mbf{A}$. \\  
$\mbf{A}[:,\mathcal{I}]$& selects the subset of columns of matrix $\mbf{A}$ with indices in set $\mathcal{I}$.\\
%$\mathcal{L}_{\tilde{\mbf{X}}_t^{\perp}}$ as $\cap_{\mbf{x} \in
%\tilde{\mbf{X}}_t} \mathcal{L}_{\mbf{x}^{\perp}}$, where $\mbf{x} \in
%\tilde{\mbf{X}}_t$ enumerates the columns of $\tilde{\mbf{X}}_t$.    - how often/where used?}\\
% $\mbf{X}_i$ & It is the $i$ row vector of matrix $\mbf{X}$. {\color{blue} subscript used differently}\\   
% $\mbf{X}_{,i}$ & It is the $i$ column vector of matrix $\mbf{X}$.  {\color{blue} seems not to be used}\\   
%\\
%$\mbf{X}\in\mathbb{R}^{m\times n_x}$ & Input data matrix to the supervised learning problem \\
%$\mbf{Y}\in\mathbb{R}^{m\times n_y}$ & Output data matrix to the supervised learning problem \\
%$m$  & The dataset size of the supervised learning problem; $\{x_i,y_i\}_{i=1}^m$\\
%$n_y$ & Full dimensionality of the output variables; $\mbf{Y}\in\mathbb{R}^{m\times n_y}$\\
%$n_x$ & Full dimensionality of the input variables; $\mbf{X}\in\mathbb{R}^{m\times n_y}$ \\
%\\ 
% & {\color{red} ok again feature vs sample notation. index differently? i and j for one, some other letters for others?}\\
%\\
%$k:\mathcal{X}\times \mathcal{X}\rightarrow \mathbb{R}$ & a kernel function \\
%$\varphi:\mathcal{X}\rightarrow\mathcal{F}_k$ & feature map from data $\mathcal{X}$ to the RKHS\\
%$\kappa:\mathcal{C}\times \mathcal{C}\rightarrow \mathbb{R}$ & a kernel function on features/variables {\color{red} C vs Rm?}\\ 
%$\phi:\mathcal{C}\rightarrow\mathcal{H}$ & feature map from the variable vector space $\mathcal{C}$ to the RKHS%\\
%$\mathcal{F}$ & RKHS
\end{tabular} 
\end{table}

The paper is organised as follows. In the next section we give overview of the method, before moving to more rigorous treatment in Section~\ref{sec:background}. There we give a brief introduction to projection operators and their matrix representation, and discuss the key detail of our approach, expressing the projector into intersection. We then move on to describing our large-scale kernelized adaptation of the algorithm in Section~\ref{sec:kernelized}. We validate our approach experimentally in Section~\ref{sec:experiments} before concluding.

% \paragraph{Notation}

% We denote scalars with lowercase letters, such as $a, \alpha$, (column) vectors with bold lowercase letters such as $\mbf{x}$, and matrices as bold uppercase letters, $\mbf{X}$. 
% % Our feature selection algorithm is developed for supervised, vector-valued learning tasks. 
% We denote the matrix containing the input data as $\mbf{X}\in\mathbb{R}^{m\times n_x}$, and output data as
% $\mbf{Y}\in\mathbb{R}^{m\times n_y}$ -- the sample size is $m$, and the number of features/variables are $n_x$ and $n_y$ for inputs and outputs, respectively.
% % 
% Some other frequently used notation used in this paper is summarized in Table~\ref{tb:notation}.

\section{Method overview}\label{sec:algprojfilter}

\begin{figure}[tb]
  \centering
  \includegraphics[width=4in,height=3in]{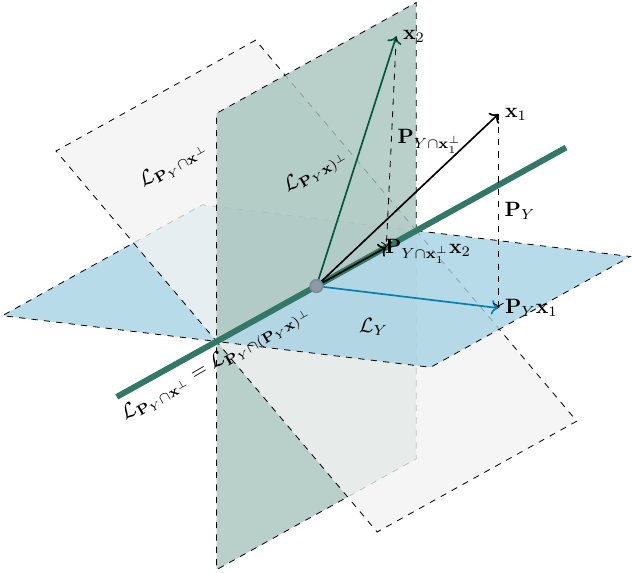}
  \caption{Illustration of the main steps of the algorithm}
  \label{fig:projective_selection}
\end{figure}

% this section is for now just copy-paste from previous introduction - will work on this later!

% It might be good to descibe the overview of the method in a new chapter after into and only have verbal desciption in the intro

%  This overview could contain the illustration and the basic ideas, plus the pseudo-code

Our algorithm is designed to perform variable selection when there are multiple dependent variables of interest.%, i.e. for vector-valued learning tasks. 
We denote the matrix containing the data from which the variables are selected as $\mbf{X}\in\mathbb{R}^{m\times n_x}$, and the reference data as
$\mbf{Y}\in\mathbb{R}^{m\times n_y}$ -- the sample size is $m$, and the number of features/variables are $n_x$ and $n_y$ (see other frequently used notation in Table~\ref{tb:notation}). Here $\mbf{X}$ and $\mbf{Y}$ could also correspond to vector-valued inputs and outputs of some supervised learning task.
Our method is based on defining correlation via projection operators: 
we define the correlation between a variable vector $\mathbf{x}\in\mathbb{R}^m$ (a column vector from $\mathbf{X}$ containing the values of a single input variable for all data points) and a
set of variables in columns of matrix $\mathbf{Y}$, as
\begin{equation}
\label{eq:subspace_correlation}
\text{corr}(\mbf{x},\mbf{Y}) =
  \left \| \mbf{P}_{\mathcal{L}_Y}\frac{\mbf{x}}{||\mbf{x}||} \right
  \| =
  \left\langle\mbf{P}_{\mathcal{L}_Y}\frac{\mbf{x}}{||\mbf{x}||},\mbf{P}_{\mathcal{L}_Y}\frac{\mbf{x}}{||\mbf{x}||}\right\rangle^{\frac{1}{2}} =
    \left\langle \frac{\mbf{x}}{||\mbf{x}||},\mbf{P}_{\mathcal{L}_Y}\frac{\mbf{x}}{||\mbf{x}||}\right\rangle^{\frac{1}{2}} 
\end{equation}  
where $ \mbf{P}_{\mathcal{L}_Y}$ (or $\mbf{P}_Y$ in shorthand) is the orthogonal projection operator into a subspace
$\mathcal{L}_Y$ spanned by the columns of $\mathbf{Y}$. 
This definition is motivated by the concept of {\textit Projection-Valued Measure} which plays a significant role in quantum mechanics theory
(see for example \cite{Nielsen2001}). % where the properties of the projection operators are exploited here, see Section \ref{sec:background} for further details.   
Our approach selects variables from input data $\mathbf{X}$ iteratively, such that correlation between the selected variable and the outputs is high, while correlation to the previously selected variables is low. 

\begin{remark}
  For sake of simplicity, we assume that for all $\mbf{x}\in\mathbb{R}^{m}$,
  $\|\mbf{x}\|=1$. 
\end{remark}

Our variable selection algorithm, \vasp{}, is illustrated in Figure~\ref{fig:projective_selection}. 
The first set of variables is is chosen simply to maximize the projection onto the subspace spanned by columns of $\mbf{Y}$, $\mathcal{L}_Y$. This is illustrated with $\mbf{x}_1$, which is projected with $\mbf{P}_Y$ as $\mbf{P}_Y\mbf{x}_1$. 
The second set of features chosen, $\mbf{x}_2$ in the figure, is projected into the intersection of $\mathcal{L}_Y$, and the orthogonal complement of the chosen feature $\mbf{x}_1$, $\mathcal{L}_{\mbf{x_1}^\perp}$. At this step, the correlation is measured with the projection operator $\mbf{P}_{\mathcal{L}_{Y} \cap \mathcal{L}_{\mbf{x}_1^{\perp}}}$. 
Interestingly, it turns out that this projected feature, $\mbf{P}_{\mbf{Y}\cap \mbf{x}_1^\perp} \mbf{x}_2$, lies also in the intersection of $\mathcal{L}_Y$ and $\mathcal{L}_{(\mbf{P}_Y \mbf{x}_1)^\perp}$.
This observation paves the way for building our efficient, recursive algorithm for the feature selection with projection operators.

The pseudo-code of the basic form of our proposed variable selection by projection, \vasp{}, algorithm is displayed in Figure~\ref{fig:selection_algorithm}. The approach is fully deterministic without randomness, and thus practical to apply. 
Similarly to CCA, our variable selection algorithm in a sense joins the variable spaces of the inputs and outputs -- both of them are considered in the same space. At the same time, in order to our selection approach to work, $\mathcal{L}_\mbf{X}$ should not be fully orthogonal $\mathcal{L}_\mbf{Y}$. 
Additionally, due to the properties of the projection operators, our approach promotes invariance: the selected explanatory variables (input features) depend only on the subspace spanned by the response variables (output features), and are independent on any transformation on the response variables that would span the same subspace. These transformations can be singular or even nonlinear, as long as they are automorphisms of the output space. 

 % \textit{Due to the properties of projection operators, our approach promotes invariance in the selection procedure: the selected explanatory variables (or input features) depend only on the subspace spanned by the response variables (output features), and are independent on any transformation on the response variables that would span the same subspace. These transformations can be singular or even nonlinear, as long as they are automorphisms of the output space. 
% This also makes our method highly persistent on certain types of noise.  [illustrated? or a small example?]
% %
% %Very often in literature the representation of the response variable is not considered but is used as is, without discussion on how the chosen representation or data transformations might affect the results. 
% %In comparison to this, our approach considers a subspace spanned by the outputs, with the projection operator. Consequently, it is invariant w.r.t. transformations preserving the subspace spanned by the outputs.  

In this basic form the algorithm is is scalable to medium-scale data, as it is limited memory required to store the projection matrix.%  requires a large amount of memory. 
In the following sections we present techniques that allow scaling to very large datasets, e.g. $m>1000000$ and $m \gg n_x,n_y$. 
A recursive representation of the projection operators (see Section 
\ref{sec:matrix_represenation}), and especially the singular vector
based form, (eq. (\ref{eq:proj-singular-repr})), significantly reduces
the demand for resources, both for memory and for computation time.

\begin{figure}[tb]
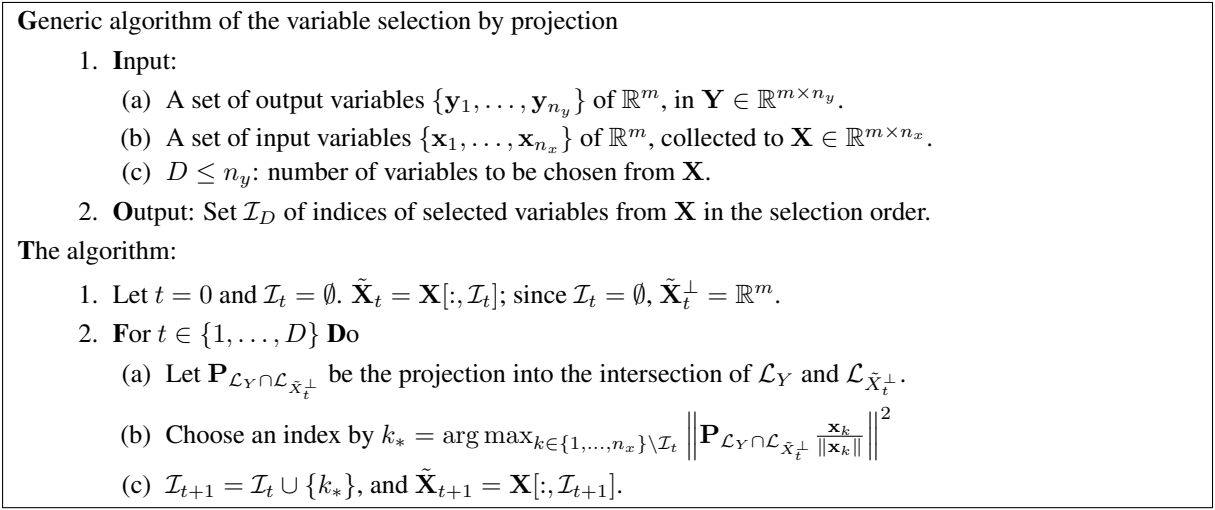

\fbox{
\begin{minipage}{0.95\linewidth}
{\textbf Generic algorithm of the variable selection by projection}
\begin{enumerate}
\item  
{\textbf Input:} 
\begin{enumerate}
\item A set of output variables $\{\mbf{y}_1,\dots,\mbf{y}_{n_y}\}$ of
  $\mathbb{R}^m$, in $\mbf{Y}\in \mathbb{R}^{m \times n_y}$. %\\The columns of $\mbf{Y}$ span the subspace $\mathcal{L}_{Y}$ 
\item A set of input variables $\{\mbf{x}_1,\dots,\mbf{x}_{n_x}\}$ of
  $\mathbb{R}^m$, collected to $\mbf{X}\in
  \mathbb{R}^{m\times n_x}$.  
\item $D \le n_y$: number of variables to be chosen from $\mbf{X}$.%is the size of the output set.  
\end{enumerate}
\item {\textbf Output:} Set $\mathcal{I}_{D}$ of indices of selected variables from $\mbf{X}$ in the selection order.% The order of the columns follows the order of the selection.    
\end{enumerate}
{\textbf The algorithm:}
\begin{enumerate}
\item Let $t=0$ and $\mathcal{I}_{t}= \emptyset$. 
$\tilde{\mbf{X}}_t = \mbf{X}[:,\mathcal{I}_t]$; since $\mathcal{I}_{t}=\emptyset$,  $\tilde{\mbf{X}}_t^{\perp} =
\mathbb{R}^{m}$.  % $\mathcal{I}_t$ is  an ordered set of indexes.  \\   
%$\mathcal{L}_{\tilde{X}_t^{\perp}}$ denotes the
%orthogonal complement of the subspace spanned by the vectors of
%$\tilde{\mbf{X}}_t = \mbf{X}[:,\mathcal{I}_t]$. If
%$\mathcal{I}_{t}=\emptyset$ then $\tilde{\mbf{X}}_t^{\perp} =
%\mathbb{R}^{m}$ is the entire space.        
 
\item {\textbf For} $t \in \{1,\dots,D\}$ {\textbf Do}
\begin{enumerate}
\item  Let 
$\mbf{P}_{\mathcal{L}_{Y} \cap \mathcal{L}_{\tilde{X}_t^{\perp}}}$
be the projection into the intersection of $\mathcal{L}_{Y}$
and  $\mathcal{L}_{\tilde{X}_t^{\perp}}$. 
\item Choose an index by 
$
k_{*} = \argmax_{k \in \{1,\dots,n_x\}\setminus \mathcal{I}_t}
\left \|\mbf{P}_{\mathcal{L}_{Y} \cap
  \mathcal{L}_{\tilde{X}_t^{\perp}}} \tfrac{\mbf{x}_k}{\|\mbf{x}_k\|}
\right \|^2
$ % there wasn't the ^2 here? but it should be, right? 
  % \begin{itemize}
  % \item Large-scale adaptation: use singular value decomposition on $\mbf{Y}$ and compute with Eq.~(\ref{eg:projection_by_parts})
  % \item Kernelized adaptation: use singular value decomposition on $\kappa(\mbf{Y},\mbf{Y})$ and compute with the iterative prodedure described in (~\ref{eq:projection_sequence})
  % \end{itemize}
\item 
$\mathcal{I}_{t+1} = \mathcal{I}_{t} \cup \{k_{*}\}$, and
$\tilde{\mbf{X}}_{t+1} = \mbf{X}[:,\mathcal{I}_{t+1}]$. 
\end{enumerate}

\end{enumerate}

\end{minipage}
}
\caption{The generic algorithm of supervised variable selection by projection}
\label{fig:selection_algorithm}
\end{figure}

\section{Projection operators}\label{sec:background}

This section first introduces relevant background on projection operators and their algebra. Then, two key points for our algorithm are discussed: the matrix representation of the projectors, and how the projection into the intersection can be expressed.

\subsection{Projection operators, projectors}

We now briefly introduce the mathematical framework describing the
projection operators of a Hilbert space. The proofs of the statements mentioned, as well as further details, are presented for example by \cite{kreyszig1989}. 

Let $\mbf{T}$ be a linear operator $\mbf{T}: \mathcal{H} \rightarrow
\mathcal{H}$. Its adjoint $\mbf{T}^{*}: \mathcal{H} \rightarrow
\mathcal{H}$ is defined by $\braket{\mbf{y},\mbf{T}^{*}\mbf{x}} =
\braket{\mbf{T}\mbf{y},\mbf{x}}$ for all $\mbf{x},\mbf{y}\in \mathcal{H}$.  A linear operator $\mbf{T}$ is
self-adjoint, or Hermitian if $\mbf{T}=\mbf{T}^{*}$, unitary if
$\mbf{T}^{*}=\mbf{T}^{-1}$  
and normal if $\mbf{T}\mbf{T}^{*} =\mbf{T}^{*}\mbf{T}$. 
On the set of self-adjoint operators of $\mathcal{H}$ one can define a
partial order $\preceq$ by 
\begin{equation}
\mbf{T}_1 \preceq \mbf{T}_2 \Leftrightarrow
\braket{\mbf{T}_1\mbf{x},\mbf{x}} \le
\braket{\mbf{T}_2\mbf{x},\mbf{x}}, 
\end{equation}  
for all $\mbf{x}\in \mathcal{H}$. 
An operator $\mbf{T}$ is positive if 
\begin{equation}
\mbf{0} \preceq \mbf{T} \Leftrightarrow 0 \le \braket{\mbf{T},\mbf{x}},
\end{equation}  
for all $\mbf{x}\in \mathcal{H}$. As a consequence we have
\begin{equation}
\mbf{T}_1 \preceq \mbf{T}_2 \Leftrightarrow \mbf{0} \preceq
\mbf{T}_2-\mbf{T}_1. 
\end{equation}  

Let $\mathcal{L}$ be a subspace of $\mathcal{H}$, the orthogonal
complement of $\mathcal{L}$ is given by
$\mathcal{L}^{\perp}=\{ \mbf{x}| \mbf{x} \perp \mbf{z}, \forall \mbf{z}
\in \mathcal{L}, \mbf{x} \in \mathcal{H} \} $.  

\begin{theorem} 
For any subspace $\mathcal{L} \subseteq \mathcal{H}$, $\mathcal{H}=
\mathcal{L}\oplus \mathcal{L}^{\perp}$. 
\end{theorem}

A linear  operator $\mbf{P}$ is a projection operator if $\mbf{P}:
\mathcal{H} \rightarrow \mathcal{L}$ for a subspace $\mathcal{L}$ of
$\mathcal{H}$. To highlight the connection between the subspace and
projection, they can be also denoted as $\mathcal{L}_{P}$ and
$\mbf{P}_{L}$.    

An operator $\mbf{P}$ is {\em idempotent} if 
% \begin{equation}
$
\mbf{P}\mbf{P}\mbf{x}=\mbf{P}\mbf{x},\ \text{or}\ \mbf{P}\mbf{P}=\mbf{P}
$
% \end{equation}
holds for any $\mbf{x}\in \mathcal{H}$.  
The projection operators can be characterized by the following statements.
\begin{theorem} 
A linear operator $\mbf{P}: \mathcal{H} \rightarrow \mathcal{H}$ is a
projection if it is self adjoint, $\mbf{P}=\mbf{P}^{*}$, and
idempotent $\mbf{P}\mbf{P} = \mbf{P}$. 
\end{theorem}

\begin{proposition}
The map connecting the set of closed subspaces\footnote{In a finite dimensional Hilbert space all subspaces are closed.} of $\mathcal{H}$ and the set of
the corresponding orthogonal projections is bijective. 
\end{proposition}
As a consequence of the idempotent and self-adjoint
properties we have that the range $\mathcal{R}(\mbf{P})$ and the null
space $\mathcal{N}(\mbf{P})$ of $\mbf{P}$ are
orthogonal, namely for any $x,y \in \mathcal{H}$ 
\begin{equation}
\label{eq:RangeNullOrth}
\braket{\mbf{P}x,y-\mbf{P}y}=\braket{\mbf{P}^2x,y-\mbf{P}y}=\braket{\mbf{P}x,(\mbf{P}-\mbf{P}^2)y}=0. 
\end{equation}

The following theorems describe some algebraic properties of projection operators we are going to
exploit. 
\begin{theorem}{\textbf (Product of projections)} 
\label{theo:product_of_projections}
Let $\mbf{P}_1$ and $\mbf{P}_2$ be projections on
$\mathcal{H}$. $\mbf{P}=\mbf{P}_1\mbf{P}_2$ is projection if and only
if $\mbf{P}_1\mbf{P}_2= \mbf{P}_2\mbf{P}_1$. Then $\mbf{P}:\mathcal{H}
\rightarrow \mathcal{L}_{P_1} \cap \mathcal{L}_{P_2}$.  
\end{theorem}

\begin{theorem}{\textbf (Sum of projections)} 
Let $\mbf{P}_1$ and $\mbf{P}_2$ be projections on
$\mathcal{H}$. $\mbf{P}=\mbf{P}_1+\mbf{P}_2$ is projection if and only
if $\mathcal{L}_{P_1} \perp \mathcal{L}_{P_2}$. Then $\mbf{P}:\mathcal{H}
\rightarrow \mathcal{L}_{P_1} \oplus \mathcal{L}_{P_2}$.  
\end{theorem}

\begin{theorem}{\textbf (Partial order)} 
Let $\mbf{P}_1$ and $\mbf{P}_2$ be projections on
$\mathcal{H}$, and $\mathcal{N}(\mbf{P}_1)$ and
$\mathcal{N}(\mbf{P}_2)$ the corresponding null spaces. Then the
following statements are equivalent.
\begin{equation}
\begin{array}{ll}
\mbf{P}_1\mbf{P}_2 &= \mbf{P}_2\mbf{P}_1 =\mbf{P}_1, \\
\mathcal{L}_{P_1} & \subseteq \mathcal{L}_{P_2}, \\ 
\mathcal{N}(\mbf{P}_1) & \supseteq \mathcal{N}(\mbf{P}_2), \\
||\mbf{P}_1\mbf{x}|| & \le ||\mbf{P}_2\mbf{x}||, \\
\mbf{P}_1 & \preceq \mbf{P}_2. \\ 
\end{array}
\end{equation}
\end{theorem}

\begin{theorem}{\textbf (Difference of projections)} 
Let $\mbf{P}_1$ and $\mbf{P}_2$ be projections on
$\mathcal{H}$. $\mbf{P}=\mbf{P}_2-\mbf{P}_1$ is projection if and only
$\mathcal{L}_{P_1} \subseteq \mathcal{L}_{P_2}$. Then $\mbf{P}:\mathcal{H}
\rightarrow \mathcal{L}_{P}$, where $\mathcal{L}_{P_2} =
\mathcal{L}_{P_1} \oplus \mathcal{L}_{P}$, namely $\mathcal{L}_{P}$ is
the orthogonal complement of $\mathcal{L}_{P_1}$ in $\mathcal{L}_{P_2}$.   
\end{theorem}
From the theorems above we can derive a simple corollary: if
$\mathcal{L}$ is a subspace, then the projection into its complement is
equal to $\mbf{P}_{\mathcal{L}^{\perp}} = \mbf{I} -  \mbf{P}_{\mathcal{L}}$.

\begin{theorem}{\textbf (Monotone increasing sequence)} 
Let $(\mbf{P}_n)$ be monotone increasing sequence of projections defined
on $\mathcal{H}$. Then:
\begin{enumerate}
\item $(\mbf{P}_n)$ is strongly operator convergent, and the limit
  $\mbf{P}$, $\mbf{P}_n \rightarrow \mbf{P}$, is a projection. 
\item $\mbf{P}: \mathcal{H} \rightarrow \cup_{n=1}^{\infty} \mathcal{L}_{\mbf{P}_n}$. 
\item $\mathcal{N}(\mbf{P}) = \cap_{n=1}^{\infty}\mathcal{N}(\mbf{P}_n)$. 
\end{enumerate}

\end{theorem}

If $\mbf{S}$ is a self-adjoint operator and $\mbf{P}$ is a projection
into the range of $\mbf{S}$ then $\mbf{S}\mbf{P}=\mbf{P}\mbf{S}$, see
\cite{JBConway1997} for further details.  

Let $\mbf{I}$ be projection into the entire space, and $\mbf{0}$ its
complement.  If $\mbf{0}\le \mbf{S} \le \mbf{I}$, and $\mbf{T}\ge \mbf{0}$
operators. If $\mbf{P}$ is a projection into $\mbf{S}+\mbf{T}$ then
$\mbf{P}$ commutes both $\mbf{S}$ and $\mbf{T}$. See in
\cite{Hayashi2002}.

% \subsubsection{Correlation between a variable and a set of
% variables}

% \label{sec:subspace_correlation}
% With the help of the projection operator we can define the
% correlations between a variable represented by a vector, $\mbf{x}$, and
% a set of variables represented by a matrix $\mbf{Y}$ whose columns span the
% subspace $\mathcal{L}_{Y}$. It is assumed that the vector, $\mbf{x}$,
% and the columns of $\mbf{Y}$ are taken from the same  Hilbert space.

% Then the correlation is given by
% \begin{equation}
%   \text{corr}(\mbf{x},\mbf{Y}) = \mbf{P}_{Y}\dfrac{\mbf{x}}{||\mbf{x}||}.
% \end{equation}
% The vectors of $\mbf{Y}$ need not be normalized since the projection
% operator in invariant on the set of vectors whenever they span the
% same subspace, see Proposition \ref{prop_unique_represenation}.  

\subsection{Matrix representation of projectors}
\label{sec:matrix_represenation}

If a basis of the Hilbert space $\mathcal{H}$ is fixed, then every
linear operator acting on $\mathcal{H}$ can be
represented by a matrix. Let the subspace $\mathcal{L}$ 
% {\color{blue} is S here on purpose, later $\mathcal{L}_A$?} 
of $\mathcal{H}$ be spanned by the vectors $\mbf{a}_1,\dots,\mbf{a}_k$ of
$\mathcal{H}$. Let us construct a matrix $\mbf{A}$ whose columns are equal to
the vectors $\mbf{a}_1,\dots,\mbf{a}_k$. Here the
linear independence of those vectors is not assumed. 
% Any matrix spanning a subspace  is also called
%a {\it frame}. 
The corresponding subspace is denoted by $\mathcal{L}_{A}$  
  
The matrix representing the orthogonal projection
operator into to subspace $\mathcal{L}_{A}$ can be expressed by a
well-known minimization problem \cite{Golub2013},
\begin{equation}
\label{eq:MinOrthProj}
\arg \min_{\mbf{w}} \|\mbf{x}-\mbf{A}\mbf{w}\|^2 =
\mbf{w}^{*} = (\mbf{A}^{T}\mbf{A})^{+}\mbf{A}^{T}\mbf{x},
\end{equation}
where $^{+}$ denotes the Moore-Penrose pseudo-inverse.
Based on eq. (\ref{eq:RangeNullOrth}) the vector $\mbf{A}^{T}\mbf{w}^{*}$
is the orthogonal projection of $\mbf{x}$ into $\mathcal{L}$. The
orthogonal projection of $\mbf{x}$ is equal to  
$
\mbf{P}_{A}\mbf{x}=\mbf{A}(\mbf{A}^{T}\mbf{A})^{+}\mbf{A}^{T}\mbf{x}$.
Since this is true for any $\mbf{x}\in \mathcal{H}$, the matrix
representation of the orthogonal projection
operator $\mbf{P}_{A}$ is given by
% $\mbf{P}_{A}=\mbf{A}(\mbf{A}^{T}\mbf{A})^{+}\mbf{A}^{T}.$
\begin{equation}
\label{eq:projoperator}
\mbf{P}_{A}=\mbf{A}(\mbf{A}^{T}\mbf{A})^{+}\mbf{A}^{T}.
\end{equation}

This formula can be simplified by exploiting the properties of the
Moore-Penrose pseudo-inverse, see for example~\cite{ABenIsrael2003}, via the singular value decomposition
$\mbf{U}_{A}\mbf{S}_{A}\mbf{V}_{A}^{T}$ of the matrix $\mbf{A}$. Here we assume that the matrix $\mbf{A} \in \mathbb{R}^{m \times n_{A}}$,
$m>n_{A}$, and $\mbf{V}_{A}$ is a square matrix, but $\mbf{U}_{A}$ contains only those left singular vectors where the corresponding singular values are not equal to zero. We have 
\begin{equation}
\label{eq:proj-singular-repr}
\fbox{$ \mbf{P}_{A} $} =\mbf{A}(\mbf{A}^{T}\mbf{A})^{+}\mbf{A}^{T} =
\mbf{A}\mbf{A}^{+} =
\mbf{U}_{A}\mbf{S}_{A}\mbf{V}_{A}^{T}\mbf{V}_{A}\mbf{S}_{A}^{+}\mbf{U}_{A}^{T}
= \fbox{$ \mbf{U}_{A}\mbf{U}_{A}^{T} $}.
\end{equation}
This representation of the projection operator plays a central role in
our variable selection algorithm.  
The following proposition ensures that the projection operator does not depend on its
representation. 
%% %%%%%%%%%%%%%%%%%%%%%%%%%%%%%%%%%%%%%%%%%%%%%%%%%%%%%%%%%%%%%%%%%%%%%
% {\color{red} Check the proof!}
\begin{proposition}
\label{prop_unique_represenation}
Assume that two different matrices $\mbf{A}$ and $\mbf{B}$ span the
same subspace $\mathcal{L}$ of dimension $k$. Then the two representations
$\mbf{P}_{A}=\mbf{U}_{A}\mbf{U}_{A}^{T}$ and
$\mbf{P}_{B}=\mbf{U}_{B}\mbf{U}_{B}^{T}$ yield the same projection
operator.  
\end{proposition}
% \begin{comment}
\begin{proof}
Since the columns of $\mbf{U}_{B}$ as linear combinations of $\mbf{B}$
are in the $\mathcal{L}$, thus
$\mbf{P}_{A}\mbf{U}_{B} = \mbf{U}_{B}$. Multiplying both sides with
$\mbf{U}_{B}^{T}$ we obtain that
$\mbf{P}_{A}\mbf{U}_{B}\mbf{U}_{B}^{T} = \mbf{U}_{B}\mbf{U}_{B}^{T}$
which is $\mbf{P}_{A}\mbf{P}_{B}=\mbf{P}_{B}$. Because the right hand
side, $\mbf{P}_{B}$,
is a projection, the left hand side $\mbf{P}_{A}\mbf{P}_{B}$ is also
one. Following the same line we have
$\mbf{P}_{B}\mbf{P}_{A}=\mbf{P}_{A}$ as well.
From Theorem \ref{theo:product_of_projections} we know that if the
product of projections is a projection, then the product of projections
is commutative,
$\mbf{P}_{B}\mbf{P}_{A}=\mbf{P}_{A}\mbf{P}_{B}$. Finally we can
conclude that
\begin{equation*}
\mbf{P}_{A}= \mbf{P}_{B}\mbf{P}_{A}=\mbf{P}_{A}\mbf{P}_{B} = \mbf{P}_{B}.
\end{equation*}     
\end{proof}
% \end{comment}

We also exploited that if $\mathcal{H}$ is finite dimensional and
the corresponding field is $\mathbb{R}$ then the adjoint of
$\mbf{P}^{*}$ is represented by the transpose $\mbf{P}^{T}$ of the
matrix $\mbf{P}$. 

% % %%%%%%%%%%%%%%%%%%%%%%%%%%%%%%%%%%%%%%%%%%%%%%%%%%%%%%%%%%%%%%%%%%%%% 
% The projection operetor depends only on its target subspace but
% independent from its matrix representation. To show this, let
% $\mbf{A}$ be a matrix spanning the subspace
% $\mathcal{L}$.
% The SVD decomposition of $\mbf{A}$ provides a orthonormal basis
% $\mbf{V}_{A}$ of $\mathcal{L}$, where the matrix
% $\mbf{U}_{A}\mbf{S}_{A}$ gives the coordinates with respect that
% basis. Let $\mbf{T}$ be any nonsingular matrix, 
% $\mbf{T}\mbf{T}^{-1}=\mbf{I}$, which can realize a transformation of
% the subspace $\mathcal{L}_{A}$, thus a transformation of the columns
% of $\mbf{A}$ as well. Now we can write up the projection operator of $\mbf{T}\mbf{A}$ 
% \begin{equation}
% \begin{array}{ll}
% \mbf{P}_{AT} =
%   \mbf{A}\mbf{T}(\mbf{T}^{T}\mbf{A}^{T}\mbf{A}\mbf{T})^{-1}\mbf{T}^{T}\mbf{A}^{T}. 
% \end{array}
% \end{equation}     
% By applying the identity
% $(\mbf{A}\mbf{T})^{-1}=\mbf{T}^{-1}\mbf{A}^{-1}$ which is also true
% in case of the Moore-Penrose invese, and the associativity of the
% matrix product, we can write 
% \begin{equation}
% \begin{array}{ll}
% \fbox{$ \mbf{P}_{TA} $} & =
%   \mbf{A}\mbf{T}\mbf{T}^{-1}(\mbf{A}^{T}\mbf{A})^{-1}(\mbf{T}^{T})^{-1}\mbf{T}^{T}\mbf{A}^{T}
%   \\
% & =  \mbf{A}(\mbf{A}^{T}\mbf{A})^{-1}\mbf{A}^{T} = \fbox{$ \mbf{P}_{A}$}, 
% \end{array} 
% \end{equation}     
% thus the projection operator is preserved after the tranformation of
% the subspace by a singular matrix. 

\subsubsection{Projection onto the intersection of subspaces -
General view}\label{sec:projections_to_intersections}

Our algorithm hinges on the orthogonal projector of the intersection of a
set of subspaces
$\{\mathcal{L}_1,\mathcal{L}_2,\dots,\mathcal{L}_{n_L}\}$. To
introduce this concept, here we mainly follow the line presented by
\cite{ABenIsrael2015}. We can start with some classical  result, first we can recall \cite{JNeumann1950}, who derived a solution in case of two
subspaces as a limit: 
\begin{equation}
\label{eq:intersect2JN}
\mbf{P}_{\mathcal{L}_1,\cap \mathcal{L}_{2}}=\lim_{n\rightarrow \infty} (\mbf{P}_{\mathcal{L}_1} \mbf{P}_{\mathcal{L}_2})^n.
\end{equation}

That result has been extended to arbitrary finite sets of subspaces by
\cite{IHalperin1962},: 
\begin{equation}
\mbf{P}_{\mathcal{L}_1,\cap \dots \cap \mathcal{L}_{n_L}} = 
\lim_{n\rightarrow \infty} (\mbf{P}_{\mathcal{L}_1} \dots \mbf{P}_{L_n})^n.
\end{equation}

\cite{WNAnderson1969}, gave an explicit formula
for the case of two subspaces by 
\begin{equation}
\mbf{P}_{\mathcal{L}_1,\cap \mathcal{L}_{2}}=2 \mbf{P}_{\mathcal{L}_1}
(\mbf{P}_{\mathcal{L}_1} + \mbf{P}_{\mathcal{L}_2})^{\dagger} \mbf{P}_{\mathcal{L}_2}.
\end{equation}  

\cite{ABenIsrael2015}, provides an alternative 
to compute $\mbf{P}_{\mathcal{L}_1,\cap \dots \cap \mathcal{L}_{n_L}}$
Here we rely on the Lemma 4.1 and Corollary 4.2 of his work:

\begin{proposition}
\label{prop:dualproj}
For $i=1,\dots,n_L$, let $\mathcal{L}_i$ be subspaces of $\mathcal{H}$,
$\mbf{P}_i$ be the corresponding projectors,
$\mbf{P}^{\perp}_i=\mbf{I}-\mbf{P}_i$, and $\lambda_i>0$. Define 
% \begin{equation}
$
\mbf{Q}:=\sum_{i=1}^{n_L}\lambda_i \mbf{P}_i^{\perp}.
$
% \end{equation}  
then  we have 
% \begin{equation}
$ \mbf{P}_{\mathcal{L}_1,\cap \dots \cap \mathcal{L}_{n_L}}=\mbf{I}-\mbf{Q}^{\dagger}\mbf{Q}.
$
% \end{equation}  
With the particular choice $\sum_{i=1}^{n_L} \lambda_i=1$, $\mbf{Q}$
might be written as   
% \begin{equation}
$ \mbf{Q}:=\mbf{I}-\sum_{i=1}^{n_L}\lambda_i \mbf{P}_i,
$
% \end{equation}  
eliminating all the complements of the projectors.
\end{proposition}

By exploting that for any projector $\mbf{P}$
$\mbf{P}^{\perp}=\mbf{I}-\mbf{P}$, the $\mbf{Q}_t$ corrsponding to
$\mbf{P}_{\mathcal{L}_{V} \cap \mathcal{L}_{\tilde{X}_t^{\perp}}}$ can
be written as
\begin{equation}
\mbf{Q}_t =\lambda_{V} (\mbf{I}-\mbf{P}_{\mathcal{L}_{V}}) + \sum_{
  \mbf{x}\in \tilde{X}_t} \lambda_x \mbf{P}_{L_{\{\mbf{x}\}}}.      
\end{equation} 

The critical point is the computation of the Moore-Penrose inverse of
$\mbf{Q}$. 

\subsection{Expressing the projector into intersection} \label{sec:projection_for_variable_selection}

To implement the proposed variable selection algorithm (Figure
\ref{sec:algprojfilter}) the projection into the intersection of an
arbitrary subspace $\mathcal{L}_{\mbf{P}}$ and the complement of an
arbitrary vector $\mbf{x}$, $\mbf{P}_{\mathcal{L} \cap
\mbf{x}^{\perp}}$,  has to be computed. The projector
$\mbf{P}_{\mathcal{L}^{\perp}}$ to the complement of a subspace
$\mathcal{L}$ can be expressed as $\mbf{I} - \mbf{P}_{\mathcal{L}}$, hence
 the projector into 
$\mbf{P}_{\mbf{x}^{\perp}}$ is given by $\mbf{I} -
\dfrac{\mbf{x}\mbf{x}^{T}}{||\mbf{x}||^2}$. Since $\mathcal{L}$ is
arbitrary we use $\mbf{P}$ instead of $\mbf{P}_{\mathcal{L}}$ for sake
of simplicity. 

While we have these two projectors, their product, according to Theorem \ref{theo:product_of_projections}, is not a projection as it does not commute:
% We face the following problem, b
%Based on Theorem
%\ref{theo:product_of_projections}, the projection into intersection
%has to a product of commutative projections, but
\begin{equation}
  \mbf{P} \left( \mbf{I} -
    \dfrac{\mbf{x}\mbf{x}^{T}}{||\mbf{x}||^2} \right) = \mbf{P} -
    \dfrac{\mbf{P}\mbf{x}\mbf{x}^{T}}{||\mbf{x}||^2}
  \ne
  \left( \mbf{I} - \dfrac{\mbf{x}\mbf{x}^{T}}{||\mbf{x}||^2}
  \right) \mbf{P}
  =\mbf{P} - \dfrac{\mbf{x}\mbf{x}^{T}\mbf{P}}{||\mbf{x}||^2},   
\end{equation}  
because in the general case $\mbf{P}\mbf{x}\mbf{x}^{T} \ne
\mbf{x}\mbf{x}^{T}\mbf{P}$. To overcome this problem we can
recognize that the intersection $\mathcal{L}_{\mbf{P}} \cap
\mathcal{L}_{\mbf{x}^{\perp}}$ can be expressed after a simple
transformation.
\begin{lemma}
\label{lemma:x_Px}
Let $\mbf{P}$ be a projector and $\mbf{x}$ be any vector, 
then the intersections $\mathcal{L}_{\mbf{P}}\cap
\mathcal{L}_{\mbf{x}^{\perp}}$ and $\mathcal{L}_{\mbf{P}}\cap
\mathcal{L}_{(\mbf{P}\mbf{x})^{\perp}}$ are the same subspaces of
$\mathcal{L}_{\mbf{P}}$.  
\end{lemma}
\begin{proof}
Any vector $\mbf{u}$ is in $\mathcal{L}_{\mbf{P}}$ if $\mbf{P}\mbf{u}
= \mbf{u}$, $\mbf{u}$ is in $\mathcal{L}_{\mbf{x}^{\perp}}$ if
$\braket{\mbf{x},\mbf{u}}=0$, and $\mbf{u}$ is in
$\mathcal{L}_{(\mbf{P}\mbf{x})^{\perp}}$ if
$\braket{\mbf{P}\mbf{x},\mbf{u}}=0$. Since $\mbf{P}\mbf{u}
= \mbf{u}$, therefore $\braket{\mbf{x},\mbf{u}} =
\braket{\mbf{P}\mbf{x},\mbf{u}}=0$.      
\end{proof}
By projecting $\mbf{x}$ into
$\mathcal{L}$ first, and then computing the corresponding
intersection, we can compute the projector into 
$\mathcal{L}_{\mbf{P}} \cap \mathcal{L}_{\mbf{x}^{\perp}}$ in a simple way.

% \begin{lemma}
% \label{lemma:lincomb}
% For any $n$, $(\mbf{P}-\mbf{P}\mbf{x}\mbf{x}^{T})^n$ can be expressed as
% linear combination of these generators:
% $\mathcal{G}=
% \{\mbf{P},\mbf{P}\mbf{x}\mbf{x}^{T},\mbf{P}\mbf{x}\mbf{x}^{T}\mbf{P}\}$.    
% \end{lemma}
% \begin{proof} In appendix, \ref{proof:lemma_P_ascombination}.  
% \end{proof}

% {\color{red} ????? Commutativity ?????}

% \subsection{Expressing the projector into intersection}

% Combining the result of Lemma \ref{lemma:lincomb} with Expression
% \ref{eq:intersect2JN} we can compute the limit of
% (\ref{eq:intersect2JN}).

\begin{proposition}
\label{prop:intersection_limit}
Let $\|\mbf{x}\|=1$ and $\alpha=\|\mbf{P}\mbf{x}\|^2=\mbf{x}^{T}\mbf{P}\mbf{P}\mbf{x}=\mbf{x}^{T}\mbf{P}\mbf{x}$. If $\alpha>0$ then 
\begin{equation}
  % \lim_{n \rightarrow \infty} (\mbf{P}-\mbf{P}\mbf{x}\mbf{x}^{T})^n
 \fbox{$  \mbf{P}_{\mathcal{L}_{\mbf{P}} \cap
 \mathcal{L}_{\mbf{x}^{\perp}}} $}
  = \mbf{P}_{\mathcal{L}_{\mbf{P}} \cap
  \mathcal{L}_{(\mbf{P}{\mbf{x})^{\perp}}}} 
  = \mbf{P}\left(\mbf{I}- \tfrac{1}{\alpha}\mbf{P}\mbf{x}\mbf{x}^{T}\mbf{P}\right)
  = \fbox{$ \mbf{P}-\tfrac{1}{\alpha}\mbf{P}\mbf{x}\mbf{x}^{T}\mbf{P} $}.
\end{equation}
When $\alpha=0$, which means $\mbf{x}$ orthogonal to
$\mathcal{L}_{\mbf{P}}$, then we have
  $ \mbf{P}_{\mathcal{L}_{\mbf{P}} \cap \mathcal{L}_{\mbf{x}^{\perp}}}  
 = \mbf{P}.$
\end{proposition}
We can check that
$\mbf{P}-\tfrac{1}{\alpha}\mbf{P}\mbf{x}\mbf{x}^{T}\mbf{P}$ is a real
projector. It is  idempotent, since
\begin{equation*}
\begin{array}{@{}l@{}}
\left(\mbf{P}-\tfrac{1}{\alpha}\mbf{P}\mbf{x}\mbf{x}^{T}\mbf{P}\right)^2 =\mbf{P}
-\tfrac{1}{\alpha}\mbf{P}\mbf{x}\mbf{x}^{T}\mbf{P}
-\tfrac{1}{\alpha}\mbf{P}\mbf{x}\mbf{x}^{T}\mbf{P}
+\tfrac{\alpha}{\alpha^2}\mbf{P}\mbf{x}\mbf{x}^{T}\mbf{P} =\ \mbf{P}-\tfrac{1}{\alpha}\mbf{P}\mbf{x}\mbf{x}^{T}\mbf{P}.
\end{array}
\end{equation*}
This agrees with Theorem \ref{theo:product_of_projections} which
states that the product of projections is a projection, idempotent and
adjoint, and it is map into the intersection of the corresponding
subspaces. 
Furthermore the orthogonality of $\mbf{x}$ and the projection of any
$\mbf{u}\in \mathcal{H}$ by
$\mbf{P}-\tfrac{1}{\alpha}\mbf{P}\mbf{x}\mbf{x}^{T}\mbf{P}$ can also be verified, as
\begin{equation*}
\begin{array}{ll}
\left\langle\mbf{x},\left(\mbf{P}-\tfrac{1}{\alpha}\mbf{P}\mbf{x}\mbf{x}^{T}\mbf{P}\right)\mbf{u}\right\rangle
=\mbf{x}^{T}\mbf{P}\mbf{u}
-\tfrac{1}{\alpha}\mbf{x}^{T}\mbf{P}\mbf{x}\mbf{x}^{T}\mbf{P}\mbf{u} 
=\mbf{x}^{T}\mbf{P}\mbf{u}-\tfrac{1}{\alpha}\alpha\mbf{x}^{T}\mbf{P}\mbf{u}
=0.
\end{array}
\end{equation*}

%% %%%%%%%%%%%%%%%%%%%%%%%%%%%%%%%%%%%%%%%%%%%%%%%%%%%%%%%%%%%%%%
%% cca related

% \begin{comment}
\begin{definition}(\cite{TTjur1984})
Two subspaces $\mathcal{L}_1$ and $\mathcal{L}_2$ are {\em orthogonal} if
for any vectors, $\mbf{x_1}\in \mathcal{L}_1$ and $\mbf{x_2}\in
\mathcal{L}_2$ $\braket{\mbf{x}_1,\mbf{x}_2}=0$ holds.
Two subspaces $\mathcal{L}_1$ and $\mathcal{L}_2$ are {\em geometrically
orthogonal} if $\mathcal{L}_1 = (\mathcal{L}_1 \cap \mathcal{L}_2)
\oplus \mathcal{C}_1$ and $\mathcal{L}_2 = (\mathcal{L}_1 \cap \mathcal{L}_2)
\oplus \mathcal{C}_2$ and $\mathcal{C}_1$ and $\mathcal{C}_2$ are
orthogonal.   
\end{definition}
\begin{lemma}(\cite{TTjur1984})
Two subspaces $\mathcal{L}_1$ and $\mathcal{L}_2$ are {\em geometrically
orthogonal} if and only if
$\mbf{P}_{\mathcal{L}_1}\mbf{P}_{\mathcal{L}_2} =
\mbf{P}_{\mathcal{L}_2}\mbf{P}_{\mathcal{L}_1}$ and
$\mathcal{L}_{\mbf{P}_{\mathcal{L}_1}\mbf{P}_{\mathcal{L}_2}} =
\mathcal{L}_1\cap \mathcal{L}_2$.  
\end{lemma}
\begin{proposition}
The subspaces $\mathcal{L}_{\mbf{P}}$ and
$\mathcal{L}_{(\mbf{P}\mbf{x})^{\perp}}$ are {\em geometrically
orthogonal}.
\end{proposition}  
\begin{proof}
It is a simple Corollary of Proposition \ref{prop:intersection_limit}.
\end{proof}
% \end{comment}

\section{Selecting variables in RKHS}\label{sec:kernelized}

In order to take into account non-linear correlations in the data, we propose a kernelized adaptation of the problem. 
Kernel methods are a group of varied machine learning models, taking
advantage of a symmetric and positive semi-definite kernel function
comparing data samples (sets of features) $k:\mathcal{X}\times
\mathcal{X}\rightarrow \mathbb{R}$. The usage of a kernel function
allows including non-linearity to the models implicitly via a feature
map $\varphi:\mathcal{X}\rightarrow\mathcal{F}_k$: a kernel evaluated
with two samples corresponds to an inner product in this so-called
feature space (more specifically reproducing kernel Hilbert space, RKHS): $k(x,z) = \langle
\varphi(x),\varphi(z)\rangle_{\mathcal{F}_k}$. %This formula is sometimes referred to as the "kernel trick". This is also the key to kernelizing traditional machine learning algorithms: if data evaluations in an algorithm are made solely via inner products, they can directly be replaced with kernel evaluations. 
For more thorough introduction to traditional kernel methods, we refer the reader e.g. to~\cite{hofmann2008kernel}.

We here propose to kernelize the variable representation.
We consider $\phi: \mathbb{R}^m
\rightarrow \mathcal{H}$, where $\mathbb{R}^m$ is the vector space containing all columns of $\mbf{Y}\in\mathbb{R}^{m\times n_y}$ and $\mbf{X}\in\mathbb{R}^{m\times n_x}$, and $\mathcal{H}$ is a RKHS. In essence, this corresponds to defining a kernel on the variable vectors, $\kappa:\mathbb{R}^m\times \mathbb{R}^m\rightarrow\mathbb{R}$ -- in fact, we assume that the $\phi$ is only given implicitly via $\kappa$.
In mathematical sense, this %matrix formed from inner products between the columns of the data 
matrix can equally well be considered to be a kernel matrix, since distinction between the rows and columns is by convention only. Usually however the matrix built from inner products between the variables is referred to as covariance operator. The covariance operators are also extended to RKHS with various applications in machine learning tasks~\cite{muandet2017kernel,minh2016approximate}. Contrary to our approach, there the feature map and kernel are defined on data space $\mathcal{X}$ instead of variable space $\mathbb{R}^m$. 
We need to mention here also the Gaussian Process Regression, \cite{10.5555/1162254} where kernels are also used to cover the covariance matrix, thus connecting the variables via inner product. % I'm not sure I get this sentence, I just copy-pasted it from a previous version 

We highlight that as the kernel is defined on variables, we can easily evaluate $\kappa(\mathbf{x}_i, \mathbf{y}_j)$.
We use the following shorthands for feature and kernel evaluations on the available training data: $\phi(\mbf{Y}) = [\phi(\mbf{y}_1),\dots,\phi(\mbf{y}_{n_y})]$ with $\phi(\mbf{y}_i), i\in[n_y]$ a column vector, and $\kappa(\mbf{Y}, \mbf{x}) = [\kappa(\mbf{y}_1, \mbf{x}),\dots,\kappa(\mbf{y}_{n_y},\mbf{x})]^\top$ a column vector of kernel evaluations (similarly for $\phi(\mbf{X})$). 
Note that $\kappa(\mbf{Y}, \mbf{Y}) = \phi(\mbf{Y})^\top \phi(\mbf{Y})$ with this notation. We further denote $\mbf{K_y} = \kappa(\mbf{Y}, \mbf{Y})$.
We assume that $\|\phi(\mbf{x})\|=1$.

\subsection{Expressing the Projection operator in RKHS}

Based on Section~\ref{sec:matrix_represenation},  Equation~(\ref{eq:proj-singular-repr}), the projection $\mbf{P}_{Y}$ is represented with the left singular vectors of $\mbf{P}_{Y}$,  $\mbf{U}_Y$. This representation is also needed for the kernelized algorithm.
However calculating directly the singular value decomposition on  $\phi(\mbf{Y})$,
$\phi(\mbf{Y})= \mbf{U}_{Y}\mbf{S}_Y\mbf{V}_Y^T$,
might not be feasible if the dimensionality of the feature space is large.  
Assuming that $\mathcal{H}$ is finite dimensional\footnote{For clarity, we restrict the discussion to finite dimensions and $\mathcal{H} = \mathbb{R}^d$ with $d<\infty$. We note that the approach is equally valid also with infinite dimensions.} with dimension $d$,  
we have $\phi(\mbf{Y}),\mbf{U}_{Y}\in \mathbb{R}^{d\times n_y}$, and $\mbf{S}_{Y}, \mbf{V}_{Y}\in \mathbb{R}^{n_y \times n_y}$. % here the dimension was m, but I don't think we want it to be the same as number of samples? 
Therefore we can write  
\begin{equation}
\mbf{S}_{Y} = 
\left [
\begin{array}{c}
\mbf{D}_{Y} \\
\sbf{\varnothing} 
\end{array} 
\right],
\mbf{D}_{Y} \in \mathbb{R}^{n_y \times n_y}, \sbf{\varnothing} \in [0]^{m-n_y,n_y},   
\end{equation}
and $\mbf{D}_{Y}$ diagonal with nonnegative elements of singular
values, thus $\phi(\mbf{Y})= \mbf{U}_{Y}\mbf{D}_{Y}\mbf{V}_Y^T$. 
Again, this decomposition can not be computed directly, however we can go on the
following line of computation.

To express the $\mbf{U}_Y$ we can apply a similar approach to what is exploited
in the computation of the kernel principal component analysis~\cite{10.1162/089976698300017467}.   
Recall that the kernel matrix on columns of $\phi(\mbf{Y})$ is $\mbf{K}_{Y} =
\phi(\mbf{Y})^{T}\phi(\mbf{Y})$. From the singular value decomposition
we can derive that $\mbf{K}_{Y} =
\mbf{V}_{Y}\mbf{S}_{Y}^{2}\mbf{V}^{T}$. This kernel has a
reasonably small size, $n_y \times n_y$, thus its eigenvalue
decomposition can be computed, which yields $\mbf{V}_{Y}$ and the
squares of the singular values of the diagonal elements of $\mbf{S}_{Y}$. By
combining these expressions we have 
\begin{equation}
\phi(\mbf{Y})\mbf{V}_{Y}\mbf{S}_{Y} = \mbf{U}_{Y}\mbf{S}_{Y}^{2}\ \Rightarrow \
\fbox{$ \mbf{U}_{Y} = \phi(\mbf{Y})\mbf{V}_{Y}\mbf{S}_{Y}^{+} $}
\end{equation}
with the help of the Moore-Penrose generalized inverse. 
Our algorithm hinges on evaluating products between projectors and the variable vectors.
We can now write the products of the $\mbf{U}_{Y}^{T}$ with an
arbitrary vector represented in $\mathcal{H}$ as
\begin{equation}
\label{eq:implicit2kernel}
\begin{array}{ll}
\fbox{$ \mbf{U}_{Y}^{T}\phi(\mbf{x}) $} =
  \mbf{S}_{Y}^{-1}\mbf{V}_{Y}^{T}\phi(\mbf{Y})^{T}\phi(\mbf{x}) 
  = \fbox{ $ \mbf{S}_{Y}^{-1}\mbf{V}_{Y}^{T} \kappa(\mbf{Y},\mbf{x})
    $}.
\end{array}
\end{equation}
Thus the product can be expressed with the help of the kernel on the variables with complexity $O(n_y^2)$ if the $\mbf{K}_{Y}$,
$\mbf{V}_{Y}$ and $\mbf{S}_{Y}^{-1}$ are precomputed. % why would this be complexity O(n^2)? if eigenvalue decomposition is needed to be calculated

\subsection{The recursive selection procedure}
 
To calculate the projection operator efficiently in each iteration we
can exploit the structure of $\mbf{P}_{\mathcal{L}_{Y} \cap
\mathcal{L}_{\tilde{X}_{t}^{\perp}}}\phi(\mbf{x})$ introduced in  Proposition
\ref{prop:intersection_limit}.   To this end, we define 
%To compute the complete iteration in the algorithm we need another,
%recursive formulation of the $ \mbf{P}_{\mathcal{L}_{Y} \cap
%\mathcal{L}_{\tilde{X}_{t}^{\perp}}}\mbf{x}$. To this end, 
an intermediate operator, projection into the complement
subspace of vector $\mbf{q}\in\mathbb{R}^n$ as:
\begin{equation}
\label{eq:proj_to_1d_complement}
\mbf{Q}(\mbf{q}) = \mbf{I} - \dfrac{\mbf{q}\mbf{q}^{T}}{||\mbf{q}||^2}.
\end{equation} 
Since $\mbf{Q}(\mbf{q})$ is a projection, we have
$\mbf{Q}(\mbf{q})=\mbf{Q}(\mbf{q})\mbf{Q}(\mbf{q})$ and $\mbf{Q}(\mbf{q}) = \mbf{Q}(\mbf{q})^{T}$. 
It can also be seen that multiplying $\mbf{Q}$ with a matrix
$\mbf{A}\in\mathbb{R}^{n\times n}$,
\begin{equation}
\mbf{Q}(\mbf{q})\mbf{A} = \left( \mbf{I} -
  \dfrac{\mbf{q}\mbf{q}^{T}}{||\mbf{q}||^2}\right) \mbf{A}  
  = \mbf{A} -
  \dfrac{\mbf{q}(\mbf{q}^{T}\mbf{A})}{||\mbf{q}||^2}, 
\end{equation}  
 has the complexity of only $O(n^2)$ since only matrix-vector and outer product are needed.  % hmm is n the best choice here? 
We are also going to use the following recurrent matrix products for a fixed $t$
\begin{equation}
\tilde{\mbf{U}}_{t} = \mbf{U}_{Y} \prod_{s=1}^{t-1} \mbf{Q}(\mbf{q}_s)
= \tilde{\mbf{U}}_{t-1}\mbf{Q}(\mbf{q}_{t-1}).  
\end{equation}
% with the help of vectors, $(\mbf{q}_1,\dots,\mbf{q}_t)$ of
% dimension $n$, and with the projector scheme defined by (\ref{eq:proj_to_1d_complement}).
Now we can write up the sequence of projections corresponding to the
Algorithm (\ref{fig:selection_algorithm}):

\smallskip  

\begin{center}
%\begin{figure}[!htb]
  \fbox{
\begin{minipage}{0.97\linewidth}  
%\begin{equation*}
$
\arraycolsep=1.2pt\def\arraystretch{1.05}
\begin{array}{@{}l@{\:}l@{}}
\text{Let} & \mbf{U}_0 = \mbf{U}_{Y},\ \mathcal{I}_0=\emptyset, \\  
\fbox{$\mbf{P}_{0}$} & =  \mbf{P}_{\phi(Y)}  
= \fbox{$ \mbf{U}_{0}\mbf{U}_{0}^{T} $}, \\
\mbf{q}_{1} & = \mbf{U}_0^{T}\phi(\mbf{x}_{k_{1*}}),\  k_{1*} =\arg\max_{k\in 
  [n_x]\setminus \mathcal{I}_{0}} ||\mbf{P}_0\phi(\mbf{x}_k)||^2,\
              \mathcal{I}_1 =\mathcal{I}_0 \cup {k_{1^*}}, \\ 
\fbox{$ \mbf{P}_1 $}& = \mbf{U}_{0}\mbf{U}_{0}^{T} 
- \dfrac{\mbf{U}_{0}\mbf{q}_1\mbf{q}_{1}^{T}\mbf{U}_{0}^{T}} 
{||\mbf{q}_1||^{2}} = \mbf{U}_{0}
 \mbf{Q}(\mbf{q}_1)\mbf{U}_{0}^{T}
 = \mbf{U}_{0}
 \mbf{Q}(\mbf{q}_1)\mbf{Q}(\mbf{q}_1)\mbf{U}_{0}^{T} 
= \fbox{$
       \mbf{U}_{1}\mbf{U}_{1}^{T} $}, 
\\
\mbf{q}_{2} & = \mbf{U}_1^{T}\phi(\mbf{x}_{k_{2*}}),\  k_{2*} =\arg\max_{k\in
  [n_x]\setminus \mathcal{I}_{1}} ||\mbf{P}_1\phi(\mbf{x}_k)||^2,\ \mathcal{I}_2 =\mathcal{I}_1 \cup {k_{2^*}}, \\
\vdots \\ 
\fbox{$ \mbf{P}_t $} & = \mbf{U}_{t-1}\mbf{U}_{t-1}^{T}  
- \dfrac{\mbf{U}_{t-1}\mbf{q}_t
  \mbf{q}_{t}^{T}\mbf{U}_{t-1}^{T}}   
{||\mbf{q}_t||^{2}} 
=  \mbf{U}_{t-1}
  \mbf{Q}(\mbf{q}_t)\mbf{U}_{t-1}^{T} \\
& = \mbf{U}_{t-1}
 \mbf{Q}(\mbf{q}_t)\mbf{Q}(\mbf{q}_t)
       \mbf{U}_{t-1}^{T} = \fbox{$
       \mbf{U}_{t}\mbf{U}_{t}^{T} $}, 
\\
\mbf{q}_{t+1} & = \mbf{U}_{t}^{T}\phi(\mbf{x}_{k_{(t+1)*}}),\
                k_{(t+1)*} =\arg\max_{k\in 
  [n_x]\setminus \mathcal{I}_{t}} ||\mbf{P}_t\phi(\mbf{x}_k)||^2,\ \mathcal{I}_{t+1} =\mathcal{I}_{t} \cup {k_{(t+1)^*}}, \\
\vdots 
\end{array}
$%\end{equation*}
\end{minipage}
}
% \caption{The recursive scheme of projection based variable selection in RKHS}
% \label{fig:projection_sequence}
% \end{figure}

\end{center}

\medskip

\begin{proposition}
\label{prop:recursive_projection}  
The sequence of projections above correctly
computes the projection operators of Algorithm
in Figure \ref{fig:selection_algorithm}.  
\end{proposition}
\begin{proof}
We apply induction on $t$ to prove the statement. In case of $t=1$ we
have by Proposition \ref{prop:intersection_limit}, that
 \begin{equation}
\begin{array}{@{}l@{\;}l@{}}
\mbf{P}_1 & = \mbf{P}_{0} -
 \dfrac{\mbf{P}_{0}\phi(\mbf{x}_{k_{1*}})\phi(\mbf{x}_{k_{1*}})^{T}\mbf{P}_{0}}
{||\mbf{P}_0\phi(\mbf{x}_{k_{1*}})||}  
 = \mbf{U}_{0}\mbf{U}_{0}^{T} 
- \dfrac{\mbf{U}_{0}\mbf{U}_{0}^{T}
 \phi(\mbf{x}_{k_{1*}})\phi(\mbf{x}_{k_{1*}})^{T}\mbf{U}_{0}\mbf{U}_{0}^{T}}  
{||\mbf{U}_{0}\mbf{U}_{0}^{T}\phi(\mbf{x}_{k_{1*})})||^{2}} \\ 
&
% = \mbf{U}_{0}\mbf{U}_{0}^{T} 
% - \dfrac{\mbf{U}_{0}\mbf{q}_{1}\mbf{q}_{1}^{T}\mbf{U}_{0}^{T}} 
% {||\mbf{q}_{1}||^{2}}  
= \mbf{U}_{0}\left( \mbf{I} -                                                       
  \dfrac{\mbf{q}_1\mbf{q}_1^{T}}{||\mbf{q}_1||^{2}} \right)\mbf{U}_{0}^{T}  
= \mbf{U}_{0}\mbf{Q}(\mbf{q}_1)\mbf{U}_{0}^{T} =
\mbf{U}_{0}\mbf{Q}(\mbf{q}_1)\mbf{Q}(\mbf{q}_1)\mbf{U}_{0}^{T}  
 =\fbox{$ \mbf{U}_{1}\mbf{U}_{1}^{T}$}. \\ 
\end{array}
\end{equation}
In transforming
$||\mbf{U}_{0}\mbf{U}_{0}^{T}\phi(\mbf{x})_{t_{1*}}||^{2}$
into $||\mbf{q}_{1}||^{2}$ we exploited that
$\mbf{U}_{0}\mbf{U}_{0}^{T}$ is a projection, hence it
is idempotent. 
Let $t>1$ be arbitrary. Suppose that 
\begin{equation}
\begin{array}{l}
\mbf{P}_t = \mbf{U}_{t-1}\mbf{U}_{t-1}^{T}  
- \dfrac{\mbf{U}_{t-1}\mbf{q}_{t}\mbf{q}_{t}^{T}\mbf{U}_{t-1}^{T}}   
{||\mbf{q}_t||^{2}} 
=  \fbox{$ \mbf{U}_t\mbf{U}_t^{T} $}, \\
\end{array}
\end{equation}
holds true. 
Now, computing the projector $t+1$ we obtain
\begin{equation*}
\begin{array}{@{}l@{\;}l@{}}
\mbf{P}_{t+1} & = \mbf{P}_{t} 
- \dfrac{\mbf{P}_{t}\phi(\mbf{x}_{k_{(t+1)*}})\phi(\mbf{x}_{k_{(t+1)*}})^{T}\mbf{P}_{t}}
{||\mbf{P}_t\phi(\mbf{x}_{k_{(t+1)*}})||^2} \\  
& = \mbf{U}_{t}\mbf{U}_{t}^{T} 
- \dfrac{\mbf{U}_{t}\mbf{U}_{t}^{T}\phi(\mbf{x}_{k_{(t+1)*}})
\phi(\mbf{x}_{k_{(t+1)*}})^{T}\mbf{U}_{t}\mbf{U}_{t}^{T}}
{||\mbf{U}_{t}\mbf{U}_{t}^{T}\phi(\mbf{x}_{k_{(t+1)*}})||^{2}} = \mbf{U}_t\left(\mbf{I} - 
\dfrac{\mbf{q}_{t+1}\mbf{q}_{t+1}}
{||\mbf{q}_{t+1}||^{2}}\right) \mbf{U}_t^{T} \\
&  = \mbf{U}_t \mbf{Q}(\mbf{q}_{t+1}) \mbf{U}_t^{T} = \mbf{U}_t \mbf{Q}(\mbf{q}_{t+1}) \mbf{Q}(\mbf{q}_{t+1})\mbf{U}_t^{T}
= \fbox{$ \mbf{U}_{t+1} \mbf{U}_{t+1}^{T} $}.
\end{array}
\end{equation*}
In the norm we again applied that
$\mbf{U}_{t}\mbf{U}_{t}^{T}$ is idempotent. 
\end{proof}
%\subsubsection{Computation}

\begin{figure}[tb]
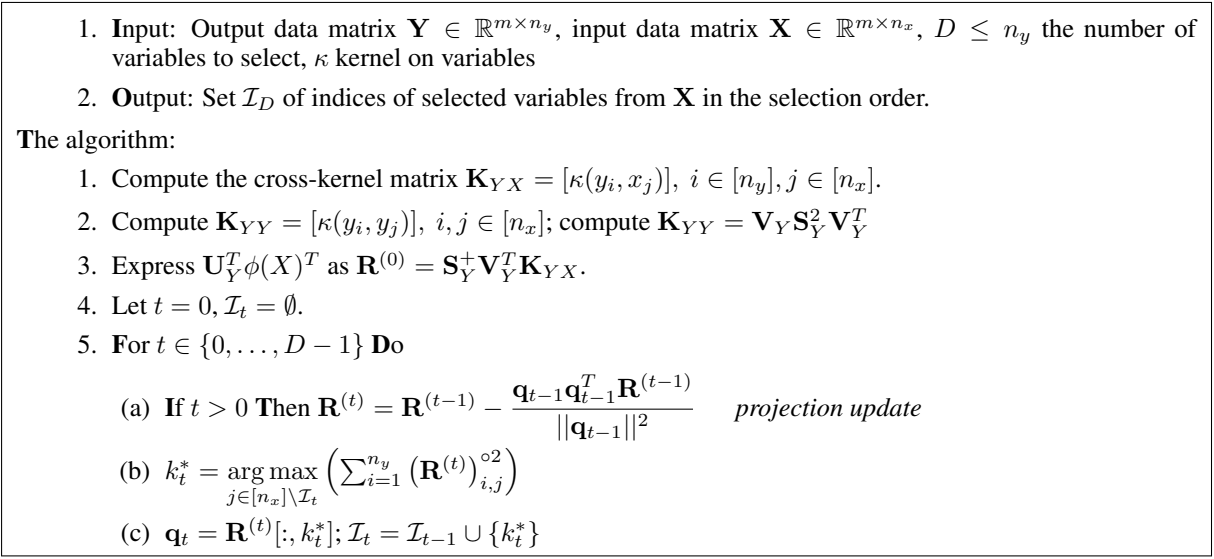

\fbox{
\begin{minipage}{0.95\linewidth}
% {\textbf Implementation of variable selection by Projection}
\begin{enumerate}
\item  
{\textbf Input:} Output data matrix $\mbf{Y}\in \mathbb{R}^{m \times n_y}$, input data matrix $\mbf{X}
  \in \mathbb{R}^{m \times n_x}$, $D \le n_y$ the number of variables to select, $\kappa$ kernel on variables  
\item {\textbf Output:} Set $\mathcal{I}_{D}$ of indices of selected variables from $\mbf{X}$ in the selection order. 
\end{enumerate}
{\textbf The algorithm:}
\begin{enumerate}
\item Compute the cross-kernel matrix $\mbf{K}_{YX} = [\kappa(y_i,x_j)],\ i \in
  [n_y], j \in [n_x]$.
\item Compute  $\mbf{K}_{YY} = [\kappa(y_i,y_j)],\ i, j \in [n_x]$; compute $\mbf{K}_{YY} = \mbf{V}_{Y}\mbf{S}_Y^2\mbf{V}_{Y}^T$  % this is also needed, no?
\item Express $\mbf{U}_{Y}^{T}\phi(X)^{T}$ as
  $\mbf{R}^{(0)}= \mbf{S}_{Y}^{+}\mbf{V}_{Y}^{T}\mbf{K}_{YX}$. 
\item Let $t=0$, $\mathcal{I}_{t}= \emptyset$.
\item {\textbf For} $t \in \{0,\dots,D-1\}$ {\textbf Do}
\begin{enumerate}
\item {\textbf If} $t>0$ {\textbf Then} $\mbf{R}^{(t)} = \mbf{R}^{(t-1)} -
  \dfrac{\mbf{q}_{t-1}\mbf{q}_{t-1}^{T}\mbf{R}^{(t-1)}}{||\mbf{q}_{t-1}||^2}$
  \ \quad  \textit{projection update}
\item  $k_t^{*} = \argmax\limits_{j \in [n_x]\setminus
  \mathcal{I}_{t}}\left(\sum_{i=1}^{n_y}\left(\mbf{R}^{(t)}\right)_{i,j}^{\circ 2}\right)$   % hmm is this clear enough? we have not opened this
  % argmax can be made prettier (but larger vertically!) by using \argmax\limits_{.....}
\item   $\mbf{q}_t= \mbf{R}^{(t)}[:,k^{*}_t]$; $\mathcal{I}_{t} = \mathcal{I}_{t-1} \cup \{k_t^{*}\}$
% \item Delete column $\mbf{R}^{(t)}[:,k^{*}_t]$  {\color{blue}\textit{if the column is deleted here, does the argmax indexing make sense anymore? in algo there is ix for this}}
\end{enumerate}

\end{enumerate}

\end{minipage}
}
\caption{Efficient implementation of the kernelized realization of supervised variable selection by projection algorithm, \vasp{}. Note the notation, e.g. $\mbf{R}^{(t)} = \mbf{U}_t\phi(X)$. }
\label{fig:selection_algorithm_final}

% about this algorithm: the U_K^(t) appears only here, while previously U_t had been used in all computations. This can be confusing to the reader. I wonder if the U_K(t) could be re-named? Perhaps Q_t would be logical? But then, Q(q) has been used for another notation. 

\end{figure}

We can express the main computation step, Step 2.b in
Algorithm \ref{fig:selection_algorithm}, by exploiting the kernelized
recursive iteration. From the sequential procedure we can see that a 
key step of the computation is the calculation of the vectors
$\mbf{q}_i$ via Equation (\ref{eq:implicit2kernel}), $\mbf{U}_{Y}^{T}\phi(\mbf{x})=\mbf{S}_{Y}^{-1}\mbf{V}_{Y}^{T}\kappa(\mbf{Y},\mbf{x})$
for an arbitrary $\phi(\mbf{x})\in \mathcal{H}$. In iteration $t$, we have 
% means computing 
% $\mbf{U}_t^{T}\phi(\mbf{x})$ as 
\begin{equation}
\begin{array}{@{}l@{}l@{}}
\mbf{q}_{t+1} = \mbf{U}_t^{T}\phi(\mbf{x}) = \left(
  \mbf{U}_{Y}\prod_{s=1}^{t-1}\mbf{Q}(\mbf{q}_s) \right)^{T}
\phi(\mbf{x})  
= \mbf{Q}(\mbf{q}_t)\cdot \dots \cdot
\mbf{Q}(\mbf{q}_1)\underbrace{\mbf{U}_{Y}^{T}\phi(\mbf{x})}_{\mbf{S}_{Y}^{-1}\mbf{V}_{Y}^{T}\kappa(\mbf{Y},\mbf{x})}. 
\end{array}
\end{equation}
Taking advantage of the recursive definition of
 $\mbf{U}_t^{T}\phi(\mbf{x})$ 
we also have that
\begin{equation}
\begin{array}{ll}
\mbf{U}_{t+1}^{T}\phi(\mbf{x}) 
& = \mbf{Q}(\mbf{q}_{t+1})\mbf{U}_{t}^{T}\phi(\mbf{x})
= \left(\mbf{I} -
\dfrac{\mbf{q}_{t+1}\mbf{q}_{t+1}^T}{||\mbf{q}_{t+1}||^2} \right)
\mbf{U}_{t}^{T}\phi(\mbf{x}),    
\end{array}
\end{equation}
where $\mbf{q}_{t+1} =
\mbf{U}_{t}^{T}\phi(\mbf{x}_{k_{(t+1)*}})$, thus all terms
relate to those computed in the previous iteration. 
The computation of the norm $||\mbf{q}_{t+1}||^2$ can also
  exploit the recursive nature  of the algorithm. 
Finally, all the feature representations $\phi(\mbf{x})$ and
$\phi(\mbf{Y})$ are implicit, and are only expressed via kernel
evaluations since they only appear in inner products.

Based on these statements and Proposition
\ref{prop:recursive_projection} we can present a concrete
implementation of our algorithm in Figure
\ref{fig:selection_algorithm_final}.  
%
%By inspecting the implementation of the algorithm, the complexity of the computation can be estimated. 
In the first step the kernels are computed, where $\mbf{K}_{YX}$ requires
$O(mn_yn_x)$, and $\mbf{K}_{Y}$ $O(mn_y^2)$ operations in case of for example linear and Gaussian kernels. For the eigenvalue
decomposition of $\mbf{K}_{Y}$ we need 
$O(n_y^3)$ operations, where $D\le \min(n_y,n_x)$. In the algorithm,
the critical step is Step 4.a. Its complexity in step $t$ is $O(n_y(n_x-t))$, thus, in general for
selecting $D$ variables we need $O(n_yn_xD)$ operations. 
Assuming that $m\gg n_y,n_x$ , the dominating part is the computation of the kernels,
thus the entire complexity is equal to $O(mn_y\max(n_x,n_y))$.

\section{Experiments}
\label{sec:experiments}

In this section we experimentally validate our approach.\footnote{The code for the algorithm is available \url{https://github.com/aalto-ics-kepaco/ProjSe}.}~\footnote{The experiments are run on a machine with this parameters: 12th Gen Intel Core$^{TM}$ i5-12600K * 10.}
We first show demonstrations on our algorithm's scalability on synthetic data, before moving on to experimenting with real data and analysing the stability of the feature selection. 

\subsection{Scalability demonstration with synthetic data}

This test is implemented via a
scheme presented by (\ref{eq:data_generation}) in
Figure~\ref{fig:selection_time_varying_sample_size} and by (\ref{eq:data_generation_full}). The components of the
input matrix, $\mbf{X}$ and the components of a transformation matrix
$\mbf{W}$ are independently sampled from normal distribution. Then
output matrix is constructed, and finally random noise is added to the output.    
\begin{equation}
  \label{eq:data_generation_full}
\begin{array}{l@{\ }c@{\ }c@{\ }c@{\ }c}
\text{Input} & & \text{Linear transformation} & & \text{Noise} \\
  \midrule 
& & \mbf{W} \sim [\mathcal{N}(0,\sigma)]^{n_x\times n_y} 
& & \mbf{E} \sim [\mathcal{N}(0,\sigma)]^{m \times n_y} \\
& & \Downarrow & & \Downarrow \\
\mbf{X} \sim [\mathcal{N}(0,\sigma)]^{m\times n_x} & \Longrightarrow  
& \mbf{Y} = \mbf{X}\mbf{W} & \Longrightarrow 
& \tilde{\mbf{Y}} = \mbf{Y}+\mbf{E}. 
\end{array} 
\end{equation}

We apply \vasp{} to this data with various sample sizes. Figure \ref{fig:selection_time_varying_sample_size} presents the
dependence of the selection time on the sample size, where the maximum sample size is
$10$ million and the number of variables is $10$ --  the variable selection is performed in less than four seconds. 
%  and Figure \ref{fig:selection_time_varying_sample_size} demonstrates
% the dependence of the selection time on the number of 
% selected variables. 
%The maximum sample size is
%$1000000$ and the number of variables is $500$. 
% \textit{not sure what will be here in the end, need to check the text again at the end}

\begin{figure}[tb]
  \begin{center}
  \begin{tabular}{@{}cc@{}}
    \begin{minipage}{0.4\linewidth}
    \includegraphics[width=2.0in]{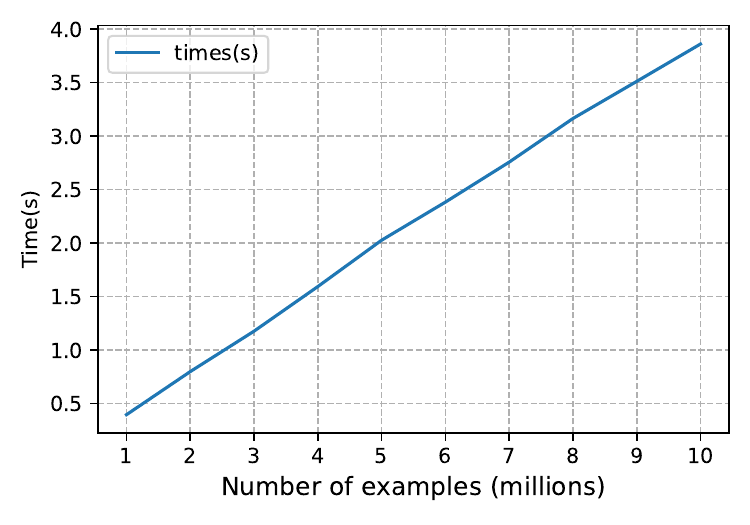}
    \end{minipage}
   &
   \begin{minipage}{0.55\linewidth}
\begin{equation}
  \label{eq:data_generation}
\renewcommand{\arraystretch}{1.2}
\begin{array}{ll}
\mbf{Y} = \mbf{X}\mbf{W} + \mbf{E} \\ \hline
\mbf{X} \sim [\mathcal{N}(0,\sigma)]^{m\times n_x},  & m = 1,\dots,10 \ \text{million},\\
\mbf{W} \sim [\mathcal{N}(0,\sigma)]^{n_x\times n_y}, & n_x = 100,  \\ 
\mbf{E} \sim [\mathcal{N}(0,\sigma)]^{m\times n_y}, & n_y = 100. \\
\end{array} 
\renewcommand{\arraystretch}{1.0}
\end{equation}
%\footnotemark[1]
  \end{minipage}
  \end{tabular}
  \end{center}
  \caption{The dependence of the variable selection time on the sample
  size is shown in seconds on the left, and the random data generation scheme applied is on right}
  \label{fig:selection_time_varying_sample_size}
\end{figure}

\subsection{Feature selection from biological data}

In this set of experiments, we compare our approach to~\cite{brouard2022feature} -- a kernel-based feature selection method, where kernels are considered traditionally on data samples instead of features.
We experiment with the two gene expression datasets, "Carcinoma", "Glioma", considered there for unsupervised feature selection.
While this kind of setting with one-view fat data is not the one our method was developed for, as the scalability we propose is first and foremost for large sample sizes, these experiments still serve for illustrative comparison of feature selection performance.

\begin{table}[tb]

\centering
\includegraphics[width=0.43\linewidth]{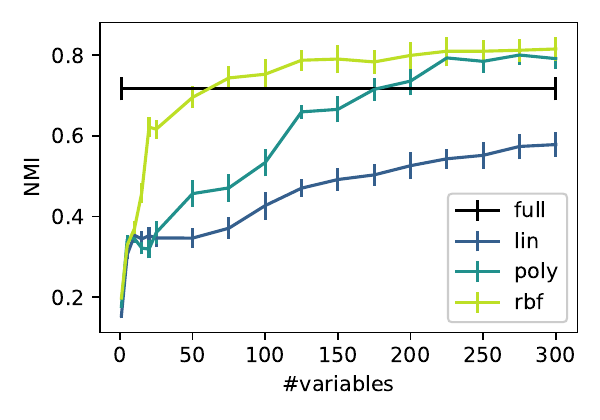}
\includegraphics[width=0.43\linewidth]{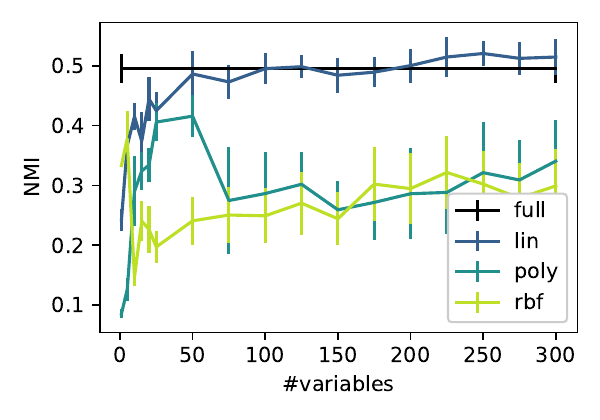}
\captionof{figure}{Clustering results (NMI) on Carcinoma (right) and Glioma (left) datasets with \vasp{} as functions of number of variables chosen, averaged over 20 runs of k-means. "Full" refers to results when full set of features is used. }\label{fig:brouard_comparison}

\captionof{table}{\vasp{} clustering results (NMI and time) with selected features (10 or 300) compared to results reported in~\cite{brouard2022feature}. Running time of \vasp{} for choosing 300 features; running time and variation of k-means is negligible. }\label{tb:brouard_comparison}

% I remember I have used a latex magic to adjust the font in a table so that the table exactly scales to the linewidth - todo figure out what that was!
% ok, it doesn't work? I wonder why... 
%\resizebox{\linewidth}{!}{
\centering
\smallskip
\footnotesize
\begin{tabular}{l|lll|lll}
\toprule 
 &\multicolumn{3}{c}{\textit{Carcinoma ($m$=174, $n_x$=9182, $C$=11)}}  & \multicolumn{3}{c}{\textit{Glioma ($m$=50, $n_x$=4434, $C$=4)}}  \\
 &NMI(10) & NMI(300) &  t (s) & NMI(10) & NMI(300) & t (s)\\
\midrule
%-    &  &   &              &   &   &         &  &  &  \\
lapl    & 0.36 (0.02)  &  0.64 (0.04)  & \textbf{0.25} (0.04)              &  \textbf{0.50} (0.03) &  0.47 (0.06)  &  \textbf{0.02} (0.00)        \\
NDFS   & 0.22 (0.28)  &  0.78 (0.03)  & 6,162 (305)             & 0.20 (0.04)  &  0.36 (0.07)  & 368 (21)          \\
UKFS  & \textbf{0.57} (0.03) &  0.75 (0.05) &  326 (52)      &  0.26 (0.05) &  0.42 (0.05) &  23.74 (4.03)     \\
\midrule
\vasp{} lin  & 0.31 (0.01) & 0.58 (0.03)  & 1,146\footnotemark[1] &   0.37 (0.01) &  \textbf{0.52 }(0.03) &  210\footnotemark[1]\\%3:29.76    \\
%${}\quad$+sup. & 0.276& 0.553   &   $\approx$1200       & \textbf{0.427}&\textbf{ 0.512} &$\approx$210       \\
\vasp{} poly &  0.34 (0.01) & 0.79 (0.03)  &  1,239\footnotemark[1]        &  0.13 (0.02) & 0.34 (0.07)  & 299\footnotemark[1]\\%4:58.955     \\
%${}\quad$+sup. & 0.341& 0.81   &          &   0.127 &0.37 &         \\
\vasp{} RBF &  0.33 (0.01) & \textbf{0.82} (0.03)  &   1,263\footnotemark[1]       &  0.38 (0.05)& 0.23 (0.06)  &  284\footnotemark[1]\\%4:43.535      \\
%${}\quad$+sup. &  0.38 & \textbf{0.851}  &          &0.206 & 0.349  &      \\
%\midrule
%koren & raw\\
%Vasp (lin) &  &   &          &   &  &       &  0.078 & 0.113 & 4.407679 \\
%${}\quad$+sup. &  &   &          &   &  &       & 0.075 & 0.179 & 4.817645 \\
%Vasp (poly) &  &   &          &   &  &       & 0.432 & 0.432 &  4.380490 \\
%${}\quad$+sup. &  &   &          &   &  &       & 0.432 & 0.432 &  4.331749 \\
%Vasp (rbf) &  &   &          &   &  &       & 0.332 & 0.432 &  4.360327  \\
%${}\quad$+sup. &  &   &          &   &  &       &  0.08 &  0.432 &  4.475745  \\
\bottomrule
\end{tabular}
% }
%
\footnotetext[1]{Running time of \vasp{} for choosing 300 features; running time and variation of k-means is negligible.}

% \end{table}

% \begin{figure}[b]
%\end{figure}
\end{table}

As the data is only available in one view in unsupervised setting, we apply our algorithm by using the data in both views: as the reference/label view and as the view the feature selection is performed on. Intuitively, this would filter out the noise and redundant features in the view.
In our method we consider linear kernel, $k(\mbf{x},\mbf{z}) = \mbf{x}^T\mbf{z}$, polynomial kernel of degree 3, $k(\mbf{x},\mbf{z}) = (\mbf{x}^T\mbf{z})^3$, and RBF kernel, $k(\mbf{x},\mbf{z}) = \exp(\|\mbf{x}-\mbf{z}\|^2/(2\sigma^2))$ with the kernel parameter $\sigma$ set as mean of pairwise distances.
We assess the performance of the feature selection by measuring the normalised mutual information~(NMI) of k-means clustering results. Here the clusterer has been given the amount of classes in the data as the number of clusters.

The results are displayed in Table~\ref{tb:brouard_comparison}, with comparison to selected methods from~\cite{brouard2022feature}: UKFS proposed there, as well as a scoring-based method "lapl"~\cite{he2005laplacian} that performed well on Glioma dataset, and NDFS~\cite{li2012unsupervised}, a clustering-based approach that performed well with Carcinoma dataset. Our method is very competitive with these, sometimes achieving better performance. As in our method the kernel is calculated on features, the running time is slower than for UKFS where the kernel on samples is considered. However notably we are still competitive when compared to the NDFS.

Additionally, Figure~\ref{fig:brouard_comparison} displays more detailed clustering results with respect to the number of variables chosen by \vasp{}.
These results also highlight the differences that can be obtained by applying different kernels on the features: with Carcinoma dataset the non-linear kernels, RBF and polynomial kernel of degree 3, are clearly superior, while with Glioma linearity works the best.

\subsection{Feature selection from vector-valued output setting}

\begin{table}[tb]

\captionof{table}{The time series classification datasets.}\label{tb:ts_data}
\centering
\begin{tabular}{lcccc}
\toprule
Dtaset  name&  \# tr. samples  & \# test samples  & \# features & \# classes\\
\midrule
Crop & 7200 &	16800  &	46 &	24\\
\makecell[l]{NonInvasiveFetalECGThorax1}&  1800 &	1965&	750 &	42   \\
ShapesAll &  600 &	600	&512 	&60\\
\bottomrule
\end{tabular}
\end{table}

%\medskip

\begin{figure}[tb]

%\centering
%{\footnotesize
%%\caption{The time series classification datasets.}\label{tb:ts_data}
%\begin{tabular}{lcccc}
%\toprule
%Dtaset  name&  \# tr. samples  & \# test samples  & \# features & \# classes\\
%\midrule
%Crop & 7200 &	16800  &	46 &	24\\
%\makecell[l]{NonInvasiveFetalECGThorax1}&  1800 &	1965&	750 &	42   \\
%ShapesAll &  600 &	600	&512 	&60\\
%\bottomrule
%\end{tabular}}

\centering
\includegraphics[width=0.32\linewidth]{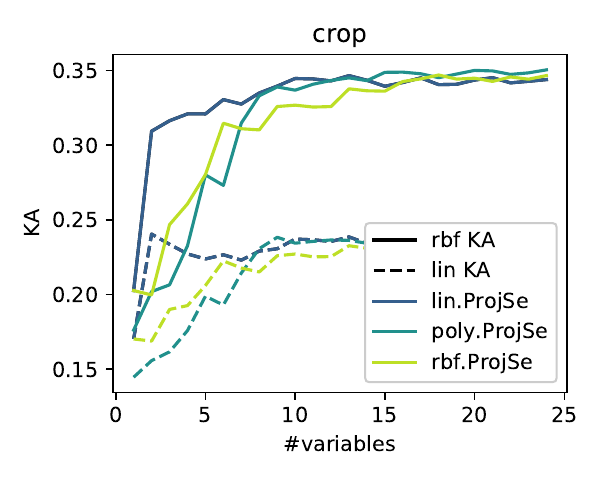}
\includegraphics[width=0.32\linewidth]{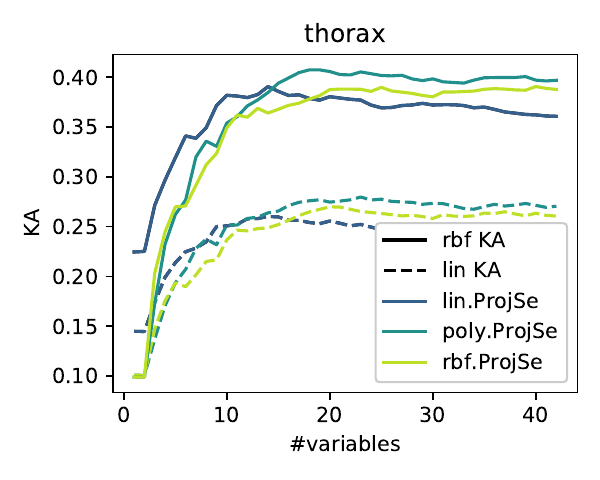}
\includegraphics[width=0.32\linewidth]{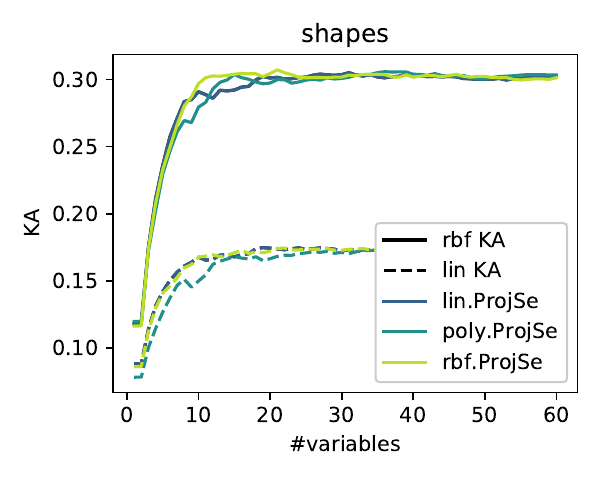}
\includegraphics[width=0.32\linewidth]{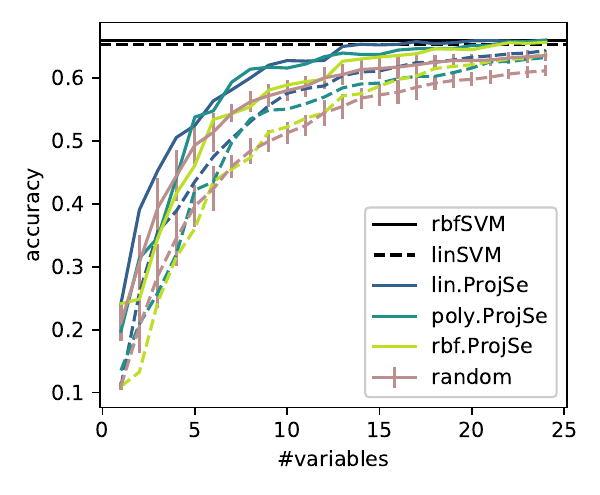}
\includegraphics[width=0.32\linewidth]{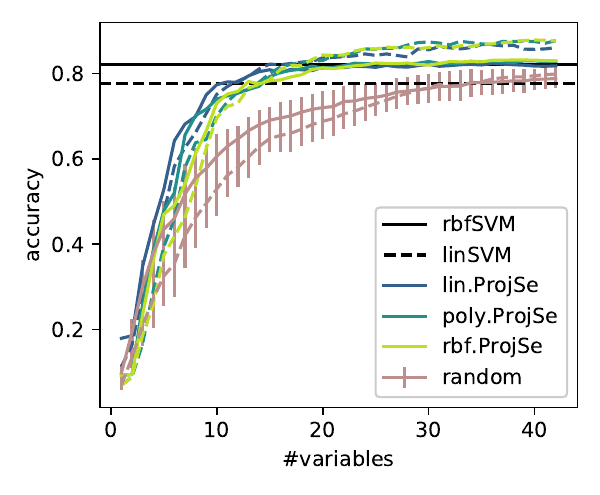}
\includegraphics[width=0.32\linewidth]{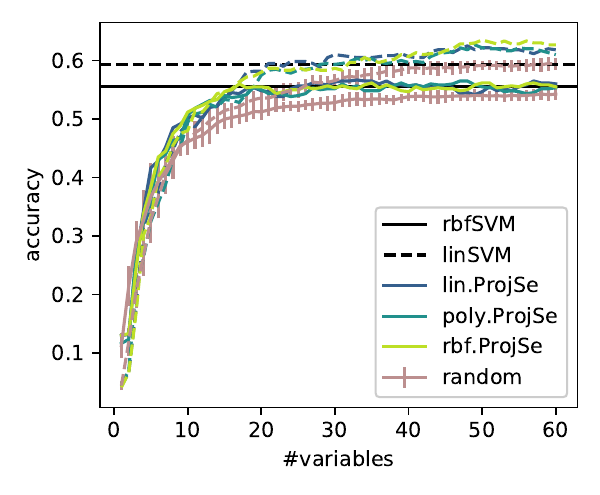}
\captionof{figure}{Results with time series datasets. The top row reports the kernel alignment of the input kernel with chosen variables (either RBF or linear) to the ideal output kernel. The bottom row reports accuracy on test set, again with both linear and RBF kernels for SVM; comparison is shown to randomly selected features and to full feature set. The colors differentiate which kernel is used on features in \vasp{}, while the line style indicates if traditional linear or RBF kernel is used on samples.}\label{fig:ka_acc_ts}
\end{figure}

%\end{table}

We next consider a setting more aligned with our method: supervised feature selection with vector-valued output as the reference view.
Here we consider three datasets from the UEA \& UCR Time Series Classification Repository\footnote{\url{http://www.timeseriesclassification.com/dataset.php}}, Crop, NonInvasiveFetalECGThorax1 ("Thorax"), and ShapesAll, as detailed in Table~\ref{tb:ts_data}. These datasets are associated with multi-class classification tasks, and we use the one-hot encoding of the labels as the vector-valued output to perform the feature selection with \vasp{}. 
As before, we consider linear, polynomial and RBF kernels. We assess the success of the feature selection task by performing classification with SVM with the selected features - here we consider both linear and RBF kernels on the data samples.

The results are displayed in Figure~\ref{fig:ka_acc_ts}, where both kernel alignment ($KA(\mbf{K},\mbf{K}') = \langle \mbf{K}_c, \mbf{K}'_c\rangle_F /(\|\mbf{K}_c\|_F\|\mbf{K}'_c\|_F)$ where $c$ denotes centering) to the linear kernel on the one-hot-encoded outputs, and accuracy of SVM classification are shown. The different kernels used in feature selection give slightly different features; however the performance on the subsequent classification task is mostly dependent on which kernel is used on the samples. Especially for Thorax and ShapesAll datasets with higher dimensionality, it can be seen that all the \vasp{} results with linear SVM outperform using the full set of features.

\subsection{Experiments with two-view data}

\begin{table}[tb]
\caption{The detailed computation times in seconds for the %complete tasks and the corresponding subprobles required in 
 variable selection method, where 10, 20, 50 and
 100 variables are extracted.} 
 \label{table:time_report}
  \centering
  \begin{tabular}{@{}c@{}c@{}}
    \begin{minipage}{0.55\linewidth}
  \begin{tabular}{@{}c@{\ }|@{\ }r@{\ }r@{\ }|@{\ }r@{\ }r@{}}
  \toprule
   & \multicolumn{2}{c}{MediaMill} & \multicolumn{2}{c}{Cifar 100} \\
    
   & Linear & RBF & Linear & RBF \\ \midrule 
 Computing $K_{yx}$ &   0.015 &   0.019 &   0.221 &   0.470 \\
Computing $K_{yy}$ &   0.007 &   0.015 &   0.006 &   0.025 \\
Eigen decomp. of $K_{yy}$ &   0.003 &   0.001 &   0.001 &   0.001 \\
Centralization of $\mathbf{K}_{xx}$ \footnotemark[1] &   0.005 &   0.014 &   1.462 & 2.065 \\ 
10 variables &   0.035 &   0.045 &   1.848 &   2.640 \\
20 variables &   0.024 &   0.048 &   1.793 &   2.657 \\
50 variables &   0.037 &   0.050 &   1.804 &   2.671 \\
100 variables &   0.049 &   0.061 &   1.833 &   2.706 \\
\bottomrule
 \end{tabular}
 \footnotetext[1]{Optional}
 \end{minipage}
 &
    \begin{minipage}{0.42\linewidth}
      \includegraphics[width=2.0in,height=1.7in]{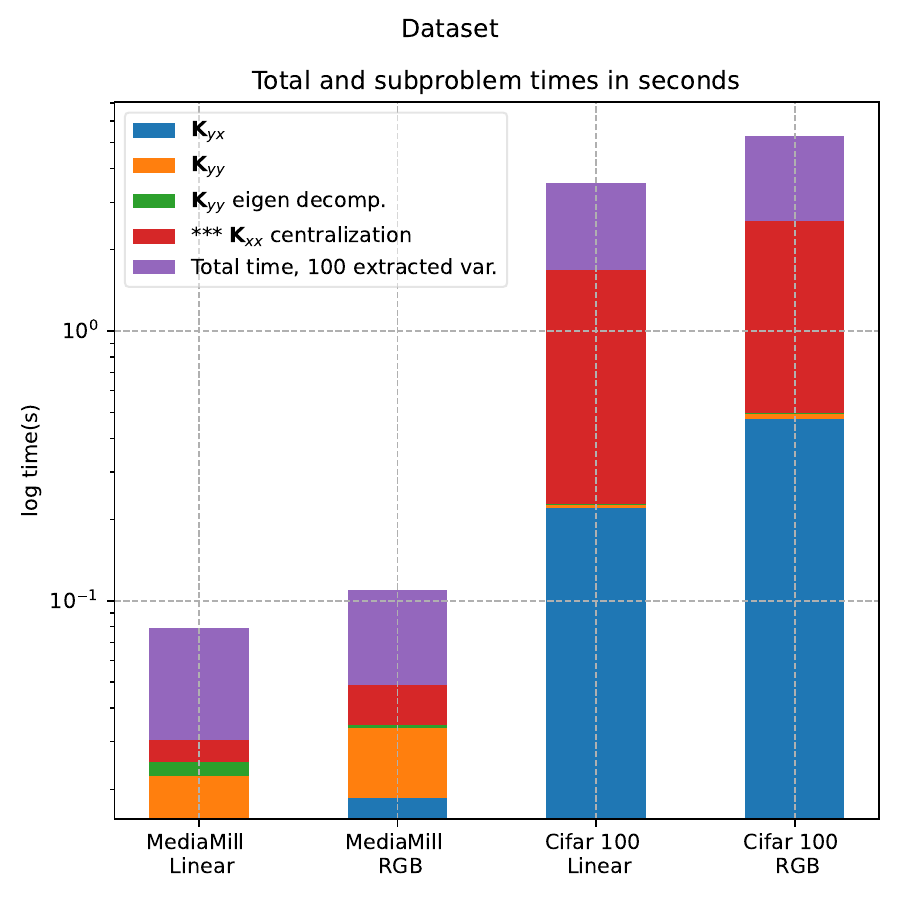}
    \end{minipage} 
 \end{tabular}
\end{table}

In our last set of experiments, we consider the following datasets:
\begin{itemize}
\item {\textbf MNIST handwritten digits} \cite{lecuny1998paper,lecuny1998data}: This dataset contains 60000 training and
  10000 test examples of handwritten digits in greyscale. 
  The number of pixels in each image is $28 \times 28 = 784$, resulting in total to 784 variables. To construct the two
  sets of variables, the image columns are split into half similarly
  as in \cite{pmlr-v28-andrew13}. Thus both views comprise of
  392 variables.   
    
\item {\textbf MediaMill dataset} \cite{10.1145/1180639.1180727}: This
  dataset contains 43907 examples which are extracted from keyframes
  of video shots. % To each keyframe 101 text annotations and 120 visual
  % features are assigned.     
  There are two views in this data: text annotations (101 variables) and visual features (120 variables).
\item {\textbf Cifar100 dataset} \cite{Krizhevsky09learningmultiple}: This
  dataset, chosen to demonstrate the scalability of \vasp{}, contains 50000 training and
  10000 test examples of color images. 
  The number of pixels in each image is $32 \times 32 = 1024$, where
  to each pixel 3 colors are assigned. The examples belong to 100
  classes, where each class contains 500 training and 100 test
  examples. The classes are represted by indicator vectors.
\end{itemize}
We perform variable selection independently on both views in the datasets. After the variable selection is performed on both sides, we compute canonical correlations between all subset pairs of the extracted variables, starting from the first ones and incrementally growing them to the entire sets. 
To demonstrate the performance of the proposed variable selection algorithm \vasp{}, it
is compared to the following methods: large-scale sparse kernel
canonical correlation analysis (GradKCCA), \cite{pmlr-v97-uurtio19a},
deep canonical correlation analysis (DCCA), \cite{pmlr-v28-andrew13},
randomized non-linear CCA (RCCA), \cite{pmlr-v32-lopez-paz14},
kernel non-linear orthogonal iterations (KNOI), \cite{Wang2015LargeScaleAK}, and
CCA through Hilbert-Schmidt independence criterion (SCCA-HSIC), \cite{8594981}.  

These CCA variants explicitly or implicitly rely on
singular value decomposition, and their performance highly depends on the
distribution of the singular values of the data matrices.  
Since the data matrices have small number of dominating singular values,
we expect from a variable selection method that it can capture a
relatively small set of variables to reproduce similar accuracy,
measured in canonical correlation. 
We need to bear in mind that the CCA-based methods add up all information represented by the full
collection of variables, however the projector based selection only
relies a relatively small subset of those variables.

% \end{comment}
\begin{figure}[tb]
  % \begin{minipage}{2.6in}
  \begin{subfigure}[b]{0.48\textwidth}
  % \parbox{2.6in}{
  \includegraphics[trim={0 0 0 2.2cm}, clip,width = 2.4in, height=2.0in]{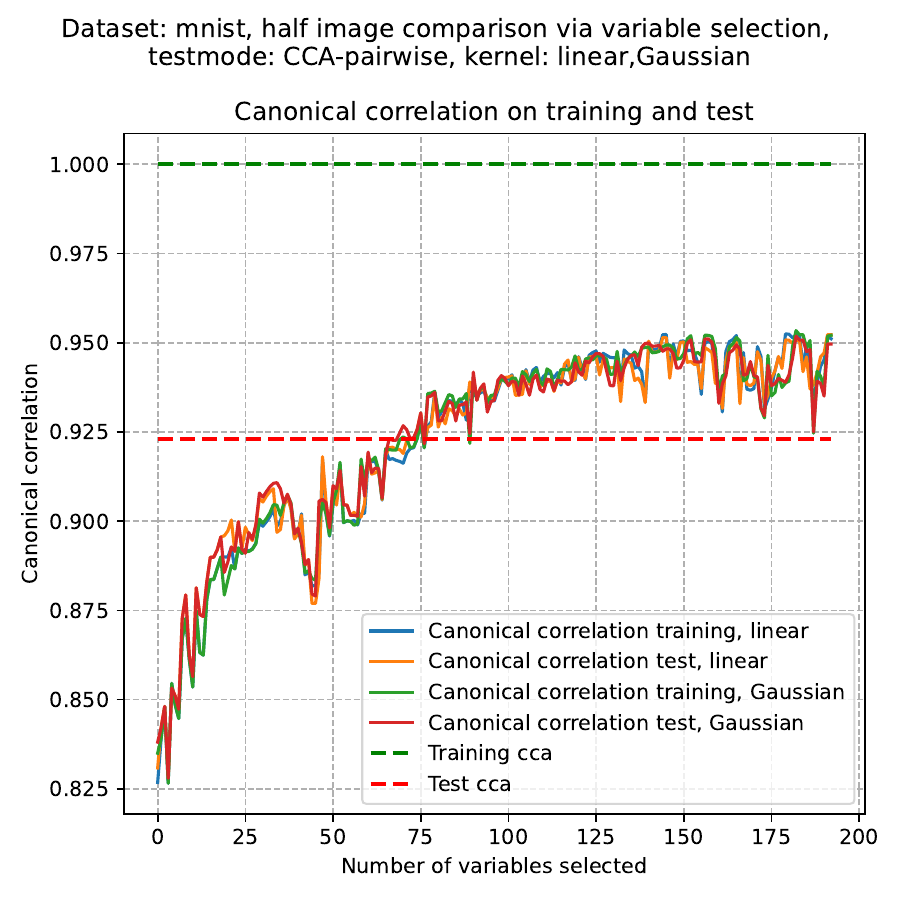}
  \caption{{ MNIST dataset}
  }
  \label{fig:mnist_cca_kernel}
  % }
  \end{subfigure}
% \end{minipage}
  \hfill
  \begin{subfigure}[b]{0.48\textwidth}
  % \parbox{2.6in}{
% \begin{minipage}{2.6in} 
  \includegraphics[trim={0 0 0 2.2cm}, clip,width = 2.4in, height=2.0in]{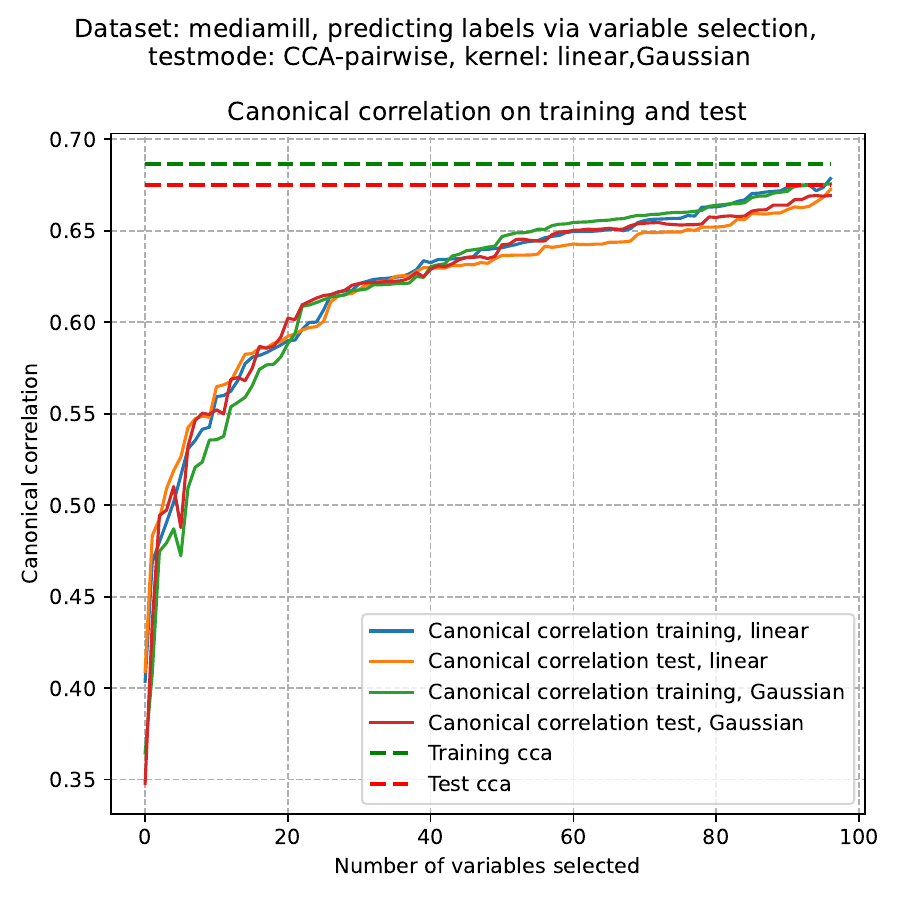}  
    \caption{{ MediaMill dataset}
  }
  
  \label{fig:mediamill_cca_kernel}
  % }
  \end{subfigure}
  \caption{Variable selection results w.r.t the number of selected variables. %{\color{red} smaller?}
  }
  \label{fig:mnist_and_mediamill}
%\end{minipage}  
\end{figure}

\begin{table}[tb]

\caption{CCA comparison on the MNIST and MediaMill datasets. }\label{tbl:cca_comparison}
\begin{tabular}{c}
\begin{minipage}{0.9\linewidth}
\centering
\begin{tabular}{c@{\quad}|l@{\quad}c@{\quad}|l@{\quad}c@{\quad}@{}}
\toprule
& \multicolumn{2}{c|}{MNIST} & \multicolumn{2}{c}{MediaMill} \\ 
& $\rho_{\text{\tiny TEST}}$ & TIME (s) & $\rho_{\text{\tiny TEST}}$ & TIME (s) \\ \midrule  
Generic CCA & 0.923 & 2.40 & 0.675 & 0.429 \\ \midrule
GradKCCA  & 0.952 $\pm$ 0.001 & 56 $\pm$6     & 0.657 $\pm$ 0.007 & 8 $\pm$4 \\
DCCA      & 0.943 $\pm$ 0.003 & 4578 $\pm$203 & 0.633 $\pm$ 0.003 & 1280 $\pm$112 \\
RCCA      & 0.949 $\pm$ 0.010 &  78 $\pm$13   & 0.626 $\pm$ 0.005 &  23 $\pm$9 \\
KNOI      & 0.950 $\pm$ 0.005 & 878 $\pm$62   & 0.645 $\pm$ 0.003 & 218 $\pm$73 \\
SCCA-HSIC & 0.934 $\pm$ 0.006 & 5611 $\pm$193 & 0.625 $\pm$ 0.002 & 1804$\pm$143
  \\ 
%  \midrule
%\makecell{\vasp{}\\10;20;50;100 var.} & \makecell{0.847;0.890;\\0.918;0.935}  & 0.45\footnotemark[1] & \makecell{0.542;0.586;\\0.631;0.672} & 0.041\footnotemark[1] \\ 
   \midrule
 \vasp{} 10 var. & 0.847 &  & 0.542 &  \\ 
 \vasp{} 20 var. & 0.890 &  & 0.586 &  \\ 
 \vasp{} 50 var. & 0.918 &   & 0.631 &  \\ 
 \vasp{} 100 var. & 0.935 & 0.45\footnotemark[1] & 0.672 & 0.041\footnotemark[1] \\ 
 \bottomrule
\end{tabular}
\footnotetext[1]{The algorithm was run once with 100 variables, the smaller amounts were subsampled }

\end{minipage}
\end{tabular}
\end{table}

Figure~\ref{fig:mnist_and_mediamill} shows the performance of \vasp{} on MNIST and MediaMill datasets with both linear and Gaussian kernels, as functions of the number of selected variables. The results are measured by canonical correlation between the subsets of variables selected from the two views. Comparing the results to the other CCA methods in Table~\ref{tbl:cca_comparison} (taken from~\cite{pmlr-v97-uurtio19a}), we observe \vasp{} obtaining comparable performance after 20 or 50 selected variables, while being orders of magnitude faster than the other methods. 
%
%The performance scores of the alternative methods in Table \ref{tbl:cca_comparison} are taken from \cite{pmlr-v97-uurtio19a}. 
Since \vasp{} is fully deterministic, there is no variance is reported for it. The heaviest computation for \vasp{} is in the beginning when eigenvalue decomposition of $\mbf{K}_{YY}$ is calculated (see Table~\ref{table:time_report}). Thus the running time varies only minimally when different number of variables is selected. 
This is also demonstrated in~\ref{table:time_report} where the running times for MNIST and Cifar100 datasets are detailed.

%\paragraph{Stability and robustness of the feature selection}
%\label{sec:stability}
    
\begin{figure}[tb]
  \centering
  \begin{tabular}{@{}cc@{}}
    \begin{minipage}{0.49\linewidth}
    \includegraphics[trim={0 0 0 1.5cm}, clip,width=0.98\linewidth]{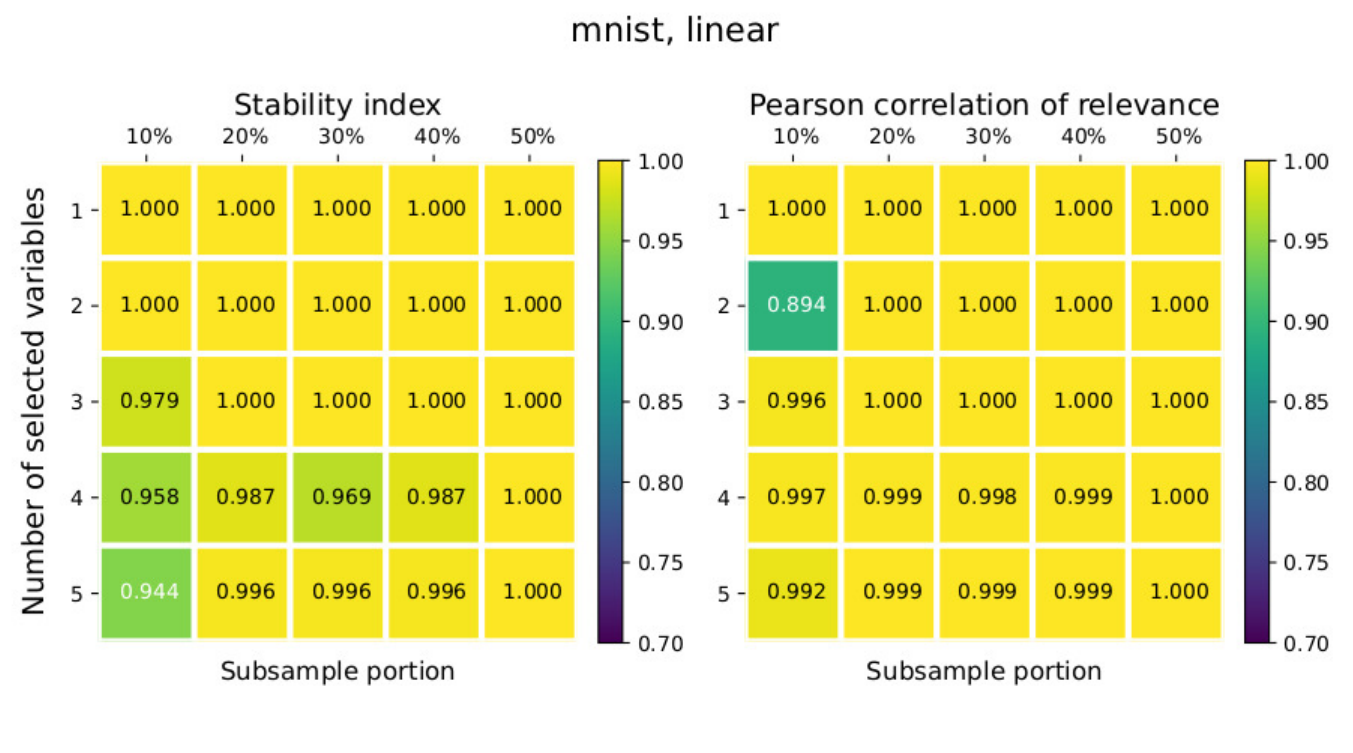}
    \end{minipage}
    &
    \begin{minipage}{0.49\linewidth}
    \includegraphics[trim={0 0 0 1.5cm}, clip,width=0.98\linewidth]{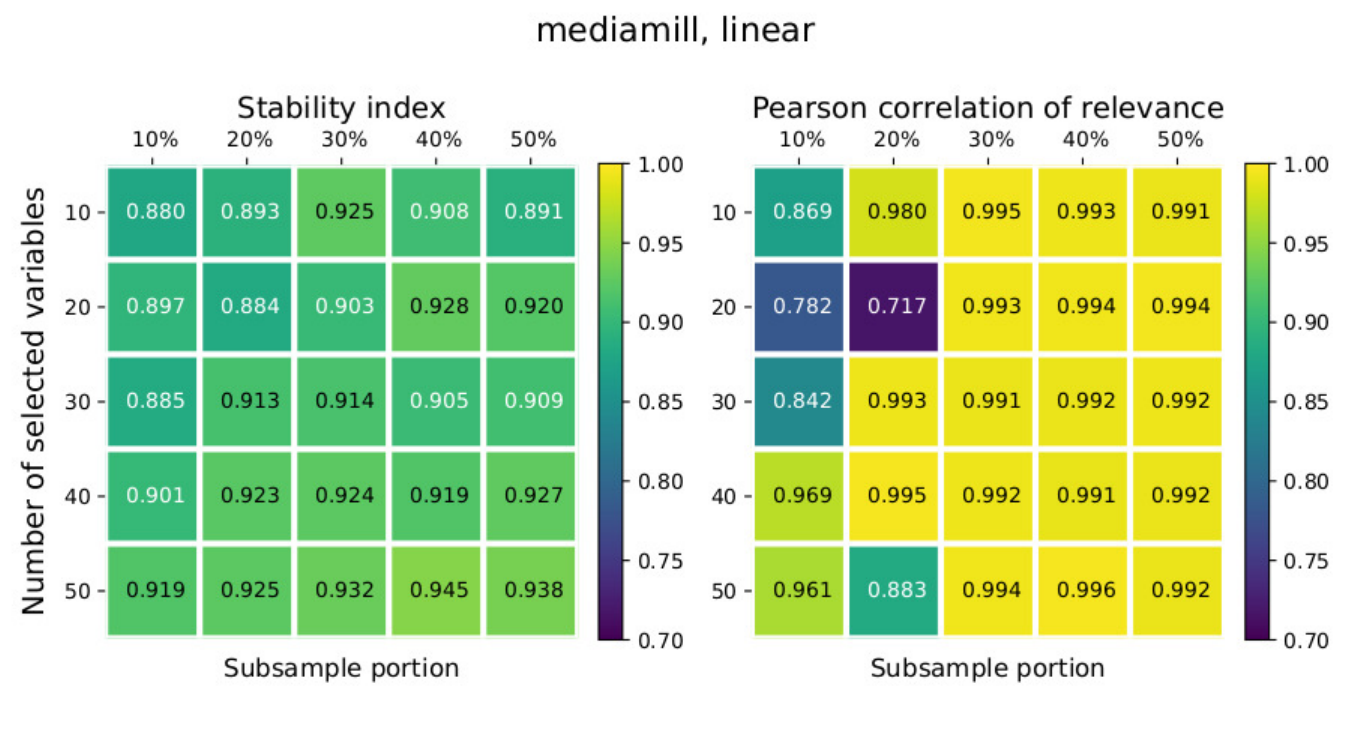}
    \end{minipage}
  \end{tabular}
\caption{Means of the two stability scores on random sub-samples computed on the data sets the MNIST (left) and the MediaMill (right), with linear kernel} 
% The first column contains the number of variables selected, and the first
% row of the size of the subsample in percentage. 
\label{table:stability_scores}  
\end{figure}

As a variable selection method, we are interested in evaluating the stability of the selection process. 
In order to measure this, we here consider the stability index~\cite{JMLR:v18:17-514}
and the average Pearson correlation of relevance~\cite{JMLR:v22:20-366}. In both measures, higher values indicate higher stability; the maximum value in both is 1. 
First the number of extracted
variables is chosen from $(1,2,\dots,5)$ in MNIST, and from $(10,20,\dots,50)$ in MediaMill. For each number of
selected variables the subsamples are taken with the following
percentage of the entire training sets: $(10\%,20\%,\dots,50\%)$.
Then random subsets are extracted $10$ times of the size given above.
The scores are computed for each number of selected
variables, and for each subsample size and finally averaged on all
random samples.   
They are shown in
Figure \ref{table:stability_scores}, where the averages
for all pairs of subsets of variables and for all subsample sizes are presented.
%{\color{blue}\textit{should there be one sentence on how this seems to be stable? }}

%\FloatBarrier
\section{Conclusion}

In this paper we introduced a novel variable selection method for two-view settings. Our method is deterministic, selecting variables based on correlation defined with projection operators. The kernelised formulation of our approach paves way for efficient and highly scalable implementation, allowing the application of our method to datasets with millions of data samples. We empirically demonstrated this efficiency and the suitability of our approach for feature selection task, with both synthetic and real data. 
% \vspace*{5cm}

\section*{Declarations}
The authors wish to acknowledge the financial support by Academy of Finland through the grants 334790 (MAGITICS), 339421 (MASF) and 345802 (AIB), as well as the Global Programme by Finnish Ministry of Education and Culture

% *****************************************************
\bibliography{projective_selection_cites}
\bibliographystyle{apalike}
% \bibliographystyle{plain}

% \appendix

\end{document}